\documentclass[12pt]{article}

\usepackage{amsthm}
\usepackage{amsmath,amssymb,amsfonts,bm,amsbsy}
\usepackage{epsfig}
\usepackage{graphicx,rotating,lscape}
\usepackage{hyperref}
\usepackage{verbatim}
\usepackage{natbib}
\usepackage{multirow,color}
\usepackage{geometry}
\usepackage{xcolor}
\usepackage[ruled]{algorithm2e}
\usepackage{moreverb}

\usepackage{subfigure}
\usepackage{setspace}

\usepackage[flushleft]{threeparttable}
\usepackage{xr}
\usepackage{float}
\newtheorem{lemma}{Lemma}       
\newtheorem{theorem}{Theorem}      

\newtheorem{proposition}{Proposition}

\newtheorem{remark}{Remark}

\def\beq{\begin{equation}}              
\def\eeq{\end{equation}}
\def\beqr{\begin{eqnarray}}             
\def\eeqr{\end{eqnarray}}               
\def\beqrs{\begin{eqnarray*}}           
\def\eeqrs{\end{eqnarray*}}     

\usepackage{moreverb}

\newcommand{\indep}{\;\, \rule[0em]{.03em}{.67em} \hspace{-.25em}
\rule[0em]{.65em}{.03em} \hspace{-.25em}
\rule[0em]{.03em}{.67em}\;\,}

\newcommand{\trans}{^\T}
\def\T{\textmd{T}}

\def\calI{\mathcal{I}}

\def\calS{\mathcal{S}}

\def\calD{\mathcal{D}}

\def\calV{\mathcal{V}}

\def \cov {\mathrm{cov}}

\def \diag {\mathrm{diag}}

\def\.{{\cdotp}}

\let\oldtabular\tabular
\renewcommand{\tabular}{\footnotesize\oldtabular}

\usepackage{moreverb}

\begin{document}

\title{\Large Distributed estimation of principal support vector machines for sufficient dimension reduction}

\author
{Jun Jin, Chao Ying, Zhou Yu\\
\\
\normalsize{School of Statistics, East China Normal University, Shanghai, China 200241}\\
\\
}

\maketitle

\begin{abstract}
The principal support vector machines method  \citep{PSVM2011} is a powerful tool for sufficient dimension reduction that replaces original predictors with their low-dimensional linear combinations without loss of information. However, the computational burden of the principal support vector machines method constrains its use for massive data. To address this issue, we in this paper propose two distributed estimation algorithms for fast implementation when the sample size is large. Both the two distributed sufficient dimension reduction estimators enjoy the same statistical efficiency as merging all the data together, which provides rigorous statistical guarantees for their application to large scale datasets. The two distributed algorithms are further adapt to prisncipal weighted support vector machines \citep{WPSVM2016} for sufficient dimension reduction in binary classification. The statistical accuracy and computational complexity of our proposed methods are examined through comprehensive simulation studies and a real data application with more than $600000$ samples.
\end{abstract}

{\bf Key Words:} Distributed estimation; Principal support vector machine; Sliced inverse regression; Sufficient dimension reduction

\newpage

\setcounter{equation}{0}

\section{Introduction}
For regression or classification problems with a univariate response variable $Y$ and a $p \times 1$ random vector
$X$, sufficient dimension reduction \citep{SIR1991, Cook:1998, Li:2018} is concerned with the scenarios where the distribution of $Y$ given $X$ depends on $X$ only through a set of linear combinations of $X$. That is, there exists a $p\times d$ matrix $\beta$ with $d\le p$, such that $$Y\indep X \mid \beta\trans X,$$ where $\indep$ stands for independence. The column space spanned by $\beta$ is called the dimension reduction subspace.  Under mild conditions \citep{YINLICOOK2008}, the intersection of all such dimension reduction subspaces is itself a dimension reduction subspace and is called the central subspace. We denote the central subspace as ${\cal S}_{Y\mid X}$ and its dimension $d=\text{dim}({\cal S}_{Y\mid X}) $ is called the structural dimension.

During past decades, a bunch of promising tools has been proposed for recovering ${\cal S}_{Y\mid X}$ from inverse regression, forward regression and semiparametric regression perspectives. As pioneered by sliced inverse regression \citep{SIR1991}, a series of inverse regression type methods were developed, which include sliced average variance estimation \citep{SAVE1991}, Contour regression\citep{CRSDR2005}, directional regression \citep{DR2007}, the inverse third moment method \citep{YIN2003}, the central kth moment method \citep{YINCOOK2002} and many others. The forward regression type methods utilized multi-index model to study ${\cal S}_{Y\mid X}$, see \cite{MAVE} and \cite{CONSTRUCTIVEMAVE}.  \cite{SEMISDR2012} and \cite{EFFICIENTSDR2013} adopt semiparametric techniques to estimate ${\cal S}_{Y\mid X}$ through solving estimating equations.

A new trend in sufficient dimension reduction is to borrow the strengths from powerful machine learning methods. The representative work is the principal support vector machines proposed by \cite{PSVM2011}, which establishes a firm connection between sufficient dimension reduction methods and the popular machine learning technique, support vector machine \citep{Vapnik:1998}. This combination inspires some further developments in sufficient dimension reduction, such as the principal weighted support vector machines \citep{WPSVM2016}, the principal $L_q$ support vector machine \citep{LQPSVM2016}, the principal minimax support vector machine \citep{MinimaxPSVM}, the penalized principal logistic regression \citep{PPLR2017}.

However, principal support vector machine can be very time consuming when one generalizes its use to nowadays massive datasets, because the core of support vector machine itself is a quadratic programming problem and the computational complexity is about $O(n^3)$ where $n$ is the sample size. In addition, large datasets are often stored across different local machines because of the data collection schemes and then data integration is extremely difficult due to communication cost, data privacy{\color{red},} and other security concerns.

To address this challenging issue, we in this paper propose two distributed estimation algorithms for principal support vector machines to facilitate its implementation with big data. For the distributed algorithms,  we
partition the $n$ data observations into $k$ subsets with equal size $m=n/k$. The naive distributed algorithm
performs principal support vector machines on each subset and then combines all the $k$ estimators suitably into an aggregated estimator. When $m\rightarrow \infty$ in the sense that $n=o(m^2)$, the aggregated estimator is  proven to be root-$n$ consistent and the resulting asymptotic variance is the same as that of the original principal support vector machines, which means that the naive divide-and-conquer approach for sufficient dimension reduction enjoys the same statistical efficiency as merging all the data together. This simple yet effective divide-and-conquer approach has also been advocated in many other statistical applications \citep{ DPCA, Lian2017,Battey2018}.

The naive distributed algorithm has its own limitation as it requires a relatively large $m$ with $n=o(m^2)$. However, some modern large-scale datasets are distributed in many local machines that can collect or store a limited amount of data. Motivated by the distributed quantile regression under such memory constraint \citep{DQR2018}, we further proposed a refined distributed estimator of ${\cal S}_{Y\mid X}$ based on an initial root-$m$ consistent estimator on a randomly selected data subset. The refined distributed  estimator is also as efficient if all data were simultaneously used to
compute the estimator without the assumption $m/ n^{1/2}\rightarrow \infty$, which provides statistical guarantees for the application of the refined distributed principal support vector machines to large scale datasets.

The principal support vector machine may fail to work for a binary response when $d\ge 2$, as it can only identify one direction in $\calS_{Y\mid X}$. To address this issue, \cite{WPSVM2016} proposed principal weighted support vector machines for sufficient dimension reduction in binary classification. And the naive and refined distributed algorithms we proposed are readily applicable to principal weighted support vector machines.

\begin{figure}
\centering     
\subfigure[Accuracy comparison with machine number $k=500$]{\label{fig:a}\includegraphics[width=70mm]{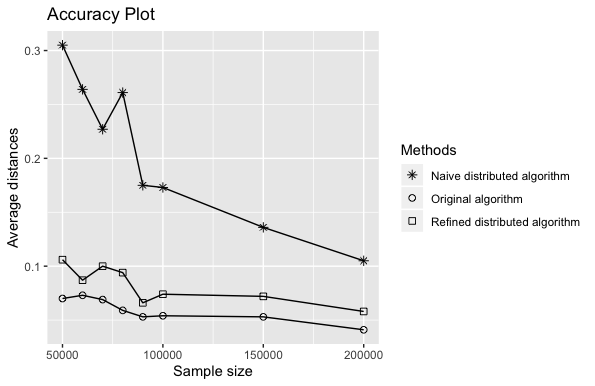}}
\subfigure[Runtime comparison with machine number $k=500$]{\label{fig:b}\includegraphics[width=70mm]{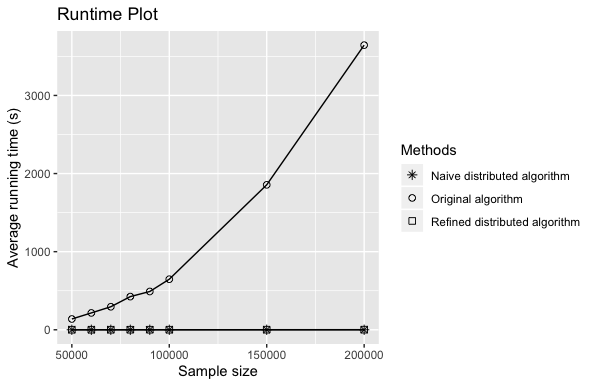}}
\caption{Average estimation errors and running times across three different methods}
\end{figure}

We investigate the performance of our proposals by simulations and a Boeing $737$ data analysis. As an illustration, we show in Figure 1 the accuracy in the estimation of the central subspace and the running time for the original method and the two distributed algorithms based on simulated Model I with $p=10$ and $k=500$.
It is obvious that the refined distributed algorithm  runs much faster than the original principal support vector machines method while retaining high accuracy  for estimating $\calS_{Y|X}$.  For the Boeing $737$ track record data during the landing process with the sample size greater than $600000$, the implementation of the original principal support vector machines will take more than $25$ hours on our personal computer. In comparison, the naive and refined distributed algorithms will only need $0.21$ and $3.54$ seconds to produce a rather satisfied sufficient dimension reduction estimator which is very close to the original estimator involving intensive computations.

\section{Principal support vector machines revisited}
Following the common practice in the literature of sufficient dimension reduction, we partition the sample space of $Y$ into $R$ non-overlapping slices. And let $\{q_{1},\ldots,q_{R-1}\}$ be the dividing points and $\tilde Y^{(\ell)}= I(Y>q_{\ell})-I(Y\le q_{\ell})$, where $\ell=1,\ldots,R-1$.  The following objective function was proposed by \cite{PSVM2011} for linear sufficient dimension reduction
\begin{align}\label{population:psvm}
L(\psi_{\ell},t_{\ell})=\psi_\ell\trans \Sigma\psi_\ell + \lambda E[1-\tilde Y^{(\ell)}\{\psi_{\ell}\trans(X-\mu)-t_{\ell}\} ]^+,
\end{align}
where $\mu = EX $ and $\Sigma=E(X-\mu)(X-\mu)^{\trans}$. Let $\theta_{0,\ell}=(\psi^{\trans}_{0,\ell},t_{0,\ell})^{\trans}$ be the minimizer of  (\ref{population:psvm}) among all $(\psi^{\trans}_{\ell},t_{\ell})^{\trans} \in \mathbb{R}^{p+1}$.
Assuming that $E(X|\beta\trans X)$ is linear in $X$, \cite{PSVM2011} further proved that
$\psi_{0,\ell}\in {\cal S}_{Y\mid X}$ for $\ell=1,\ldots,R-1$. The population level candidate matrix of the linear principal support vector machines is then constructed as
\begin{align}\label{eq:M0}
M_0 = \sum_{\ell=1}^{R-1}  \psi_{0,\ell}\psi^{\trans}_{0,\ell}.
\end{align}
The top $d$ eigenvectors $\calV_0=(\nu_1,\ldots,\nu_k) $ of $M_0$  provide a basis of the central subspace ${\cal{S}}_{Y\mid X}$.

Given a random sample $\{(X_i, Y_i), i=1,\ldots,n\}$ from $(X,Y)$, we can estimate $\mu$ and $\Sigma$ through $\hat \mu=E_n(X)$ and $\hat\Sigma=E_n\{(X-\hat \mu)(X-\hat \mu)^\T\}$, where $E_n(\cdot)$ indicates the sample average $n^{-1}\sum_{i=1}^n (\cdot)$. Then the sample version of (\ref{population:psvm}) is
\begin{align}\label{sample:psvm}
\hat L(\psi_{\ell},t_{\ell})=\psi_\ell\trans \hat \Sigma\psi_\ell + \lambda E_n\{1-\tilde Y^{(\ell)}[\psi_{\ell}\trans(X-\hat\mu)-t_{\ell}] \}^+.
\end{align}
Denote $\hat\theta_{n,\ell}=(\hat\psi^{\trans}_{n,\ell},\hat t_{n,\ell})^{\trans}$ as the corresponding minimizer. Then the sample level candidate matrix is
\begin{align}\label{eq:Mn}
\widehat M_n = \sum_{\ell=1}^{R-1}  \hat\psi_{n,\ell}\hat\psi^{\trans}_{n,\ell}.
\end{align}
And the first $d$ eigenvectors of $\widehat M$, denoted by $\widehat \calV_n=(\hat \nu_1,\ldots,\hat \nu_k) $, forms an estimate of the central subspace ${\cal{S}}_{Y\mid X}$.

We begin with some notations to present the asymptotic results of the principal support vector machines.
Let $\tilde X=(X\trans - \mu\trans,-1)\trans$,  $Z^{(\ell)}=(X \trans,\tilde Y^{(\ell)})\trans$.
Denote $D$ by the $d\times d$ diagonal matrix with its diagonal elements being the nonzero eigenvalues of $M_0$. Let $\Gamma$ be the $p\times d$ matrix whose columns are the eigenvectors of $M_0$ corresponding to the nonzero eigenvalues. We define
\begin{align*}
&D_{\theta_{0,\ell}} (Z^{(\ell)})= (2\psi_{0,\ell}\trans \Sigma,0)\trans/\lambda- \{\tilde X \tilde Y^{(\ell)} I(1-\theta_{0,\ell}\trans\tilde X\tilde Y^{(\ell)}>0)\},\\
&H_{\theta_{0,\ell}}=2\diag(\Sigma,0)/\lambda+\sum_{\tilde y=1,-1} P(\tilde Y^{(\ell)}=\tilde y) f_{\psi_{0,\ell}\trans X|\tilde Y^{(\ell)} }(t_\ell+\tilde y |\tilde y) E(\tilde X \tilde X\trans |\psi_{0,\ell}\trans X= t_\ell+\tilde y),
\end{align*}
where $\diag(\Sigma,0)$ denotes the $(p + 1) \times (p + 1)$ block-diagonal matrix whose block-diagonal
elements are $\Sigma$ and $0$, and $f_{\psi_{0,\ell}\trans X|\tilde Y^{(\ell)} }$ is the conditional density function of $\psi_{0,\ell}\trans X$ given $\tilde Y^{(\ell)}$.
In addition,  let
$ S_{\theta_{0,\ell}}(Z^{(\ell)}) = - H^{-1}_{\theta_{0,\ell}}D_{\theta_{0,\ell}} (Z^{(\ell)})$ and
$ \Lambda_{rt} = E\{S_{\theta_{0,r}}(Z^{(r)})  S\trans_{\theta_{0,t}}(Z^{(t)}) \}$. \cite{PSVM2011} established the asymptotic property of principal support vector machines as follows.


\begin{theorem}\label{theo: asy psvm} Assume the regularity conditions 1-5 listed in the Appendix, then
\begin{align*}
&n^{1/2}\textup{vec}(\widehat M_n -M_{0})\rightarrow N(0_{p^2},\Sigma_{M}),\\
&n^{1/2}\textup{vec}(\widehat \calV_{n}-\calV_{0})\rightarrow N(0_{pd},\Sigma_{V}),
\end{align*}
in distribution, where 
\begin{align*}
\Sigma_{M}=& (I_{p^2}+K_{p,p}) \sum_{r=1}^{R-1}\sum_{t=1}^{R-1} (\psi_{0,r}\psi\trans_{0,t}\otimes \Lambda_{rt}) (I_{p^2}+K_{p,p}),\\
 \Sigma_{V}=&( D^{-1}\Gamma\trans\otimes I_p) \Sigma_M(\Gamma D^{-1}\otimes I_p),
\end{align*}
and $K_{p,p} \in \mathbb{R}^{p^2\times p^2}$ denotes the communication matrix satisfying $K_{p,p} \textup{vec}(A)= \textup{vec}(A\trans)$ for a matrix
$A \in \mathbb{R}^{p\times p}$.
\end{theorem}
However, as $R-1$ support vector machines are involved in the above estimation procedure, the principal support vector machine is very computational intensive when $n$ is large. We in the next propose two distributed algorithms for fast computation while enjoying the same asymptotic property.

\section{Naive distributed estimation}
To design the naive distributed algorithm of principal support vector machines, we randomly and evenly partitions the data sample  $\calD=\{(X_1,Y_1),\ldots,(X_n,Y_n)\}$ into $k$ disjoint subsets $\calD_1,\ldots,\calD_k$, such that $\calD=\cup_{j=1}^k\calD_j$ and $\calD_i\cup\calD_j=\varnothing$ for $1\le i\neq j\le k$. Without loss of generality, assume that $n$ can be divided evenly and hence $m=n/k$. Let $\calI_j \subset \{1,\ldots,n\}$ be the index set corresponding to $\calD_j$. Then for each batch of data $\calD_j$, we can estimate $\mu$ and $\Sigma$ as $\hat \mu_j= m^{-1}\sum_{i\in\calI_j} X_i$ and $\hat\Sigma_j= m^{-1}\sum_{i\in\calI_j} (X_i-\hat\mu_j)(X_i-\hat\mu_j)\trans$. In addition, the sample version of the objective function (\ref{population:psvm}) based on the $j$th batch of data $\calD_j$
becomes
\begin{align*}
\hat L_j(\psi_{\ell},t_{\ell})=\psi_\ell\trans \hat \Sigma_j\psi_\ell + \lambda m^{-1}\sum_{i\in\calI_j}[1-\tilde Y_i^{(\ell)}\{\psi_{\ell}\trans(X_i-\hat\mu_j)-t_{\ell}\} ]^+.
\end{align*}
Let $\hat\theta_{j,\ell}=(\hat\psi^{\trans}_{j,\ell},\hat t_{j,\ell})^{\trans}$ be the corresponding minimizer on $\calD_j$. The resulting sample level candidate matrix constructed based on $\calD_j$ is then
\begin{align}\label{eq:Mj}
\widehat M_j = \sum_{\ell=1}^{R-1}  \hat\psi_{j,\ell}\hat\psi^{\trans}_{j,\ell}.
\end{align}
Finally, the aggregated estimator is defined by
\begin{align}\label{eq:Mj}
\widetilde M = \sum_{j=1}^{k} \widehat M_j/k .
\end{align}
And then the leading $d$ eigenvectors of $\widetilde M$, denoted by $\widetilde V=(\tilde \nu_1,\ldots,\tilde \nu_d) $, are the naive distributed estimators of the central subspace $\calS_{Y|X}$. And the asymptotic property is given below.

\begin{theorem}\label{theo: asy npsvm} In addition to the regularity conditions 1-6 listed in the Appendix, assume that $m\rightarrow \infty$ and  $k\rightarrow\infty$ such that $n=o(m^{2\gamma})$ where $1/2<\gamma\le 1$ is a positive constant specified in condition 6 in the Appendix. Then we have
	\begin{align*}
	&n^{1/2}\textup{vec}(\widetilde M-M_{0})\rightarrow N(0_{p^2},\Sigma_{M}),\\
	&n^{1/2}\textup{vec}(\widetilde \calV-\calV_{0})\rightarrow N(0_{pd},\Sigma_{V}),
	\end{align*}
	in distribution.
	\end{theorem}

The naive distributed algorithm of principal support vector machines requires $n=o(m^2)$ to achieve the same asymptotic  efficiency as the original method. In other words, the naive distributed estimator may not work well when the batch size $m$ is relative small compared to the number of batches $k$. In the next section, we will propose a refined distributed algorithm which does not need such a stringent condition.

\section{Refined distributed estimation}
Inspired by the smoothing technique introduced in  \cite{DQR2018} and \cite{DSVM} , we consider a smooth version of (\ref{sample:psvm}) instead, that is
\begin{align}\label{smooth loss}
\hat L(\psi_{\ell},t_{\ell})=\psi_\ell\trans \hat \Sigma\psi_\ell + \lambda n^{-1} \sum_{i=1}^n K_h[1-\tilde Y_i^{(\ell)}\{\psi_{\ell}\trans(X_i-\hat\mu)-t_{\ell}\} ].
\end{align}
Here the hinge loss function ${u}^+=\max(u,0)$ is approximated by the smooth function $K_h(u)=u H(u/h)$ as the bandwidth $h$ tends to zero and
$ H(u)$ is a smooth and differentiable function satisfying $H(u) = 1$ when
$u\ge 1$ and $H(u) = 0$ when $u \le -1$. Moreover, in this paper we assume that $H$ holds $C-$Lipschitznss for some constant $C$. Let $\hat X_i=(X_i\trans-\hat\mu\trans,-1)\trans$ and $g(\hat X_i,\tilde Y_i^{(\ell)},\theta_{\ell})=1-\tilde Y_i^{(\ell)}\theta_{\ell}\trans\hat X_i$. Then the optimal $\theta_{\ell}=(\psi_{\ell}\trans,t_\ell)\trans$ that minimizes (\ref{smooth loss}) should be the solution of the following equations:
\begin{align*}
\lambda n^{-1}\sum_{i=1}^n \hat X_i\tilde Y^{(\ell)}_i
[H\{g(\hat X_i,\tilde Y_i^{(\ell)},\theta_{\ell})/h\}+
\{g(\hat X_i,\tilde Y_i^{(\ell)},\theta_{\ell})/h\}H^\prime\{g(\hat X_i,\tilde Y_i^{(\ell)},\theta_{\ell})/h\}]=2\diag(\hat\Sigma,0)\theta_{\ell}
\end{align*}

After some rearrangements, we have
\begin{align}\nonumber
\theta_{\ell}=& \big[ \lambda n^{-1}\sum_{i=1}^n \hat X_i\hat X_i\trans
H^\prime\{g(\hat X_i,\tilde Y_i^{(\ell)},\theta_{\ell})/h\}/h+2\diag(\hat\Sigma,0)\big]^{-1}
\\\label{eq: smooth_theta}
&\big[\lambda n^{-1}\sum_{i=1}^n \hat X_i\tilde Y^{(\ell)}_i
\{H\{g(\hat X_i,\tilde Y_i^{(\ell)},\theta_{\ell})/h\}+H^\prime\{g(\hat X_i,\tilde Y_i^{(\ell)},\theta_{\ell})/h\}/h\}\big].\\\nonumber
\end{align}
Given a good initial value $\hat\theta_{(0),\ell}=(\hat\psi\trans_{(0),\ell},\hat t_{(0),\ell})\trans$, we can adopt (\ref{eq: smooth_theta}) to update $\theta_{\ell}$ as follows:
\begin{align}\nonumber
\hat\theta_{(1),\ell}=& \big[ \lambda n^{-1}\sum_{i=1}^n \hat X_i\hat X_i\trans
H^\prime\{g(\hat X_i,\tilde Y_i^{(\ell)},\hat\theta_{(0),\ell})/h\}/h+2\diag(\hat\Sigma,0)\big]^{-1}
\\\label{eq: update_theta}
&\big[\lambda n^{-1}\sum_{i=1}^n \hat X_i\tilde Y^{(\ell)}_i
\{H\{g(\hat X_i,\tilde Y_i^{(\ell)},\hat\theta_{(0),\ell})/h\}+H^\prime\{g(\hat X_i,\tilde Y_i^{(\ell)},\hat\theta_{(0),\ell})/h\}/h\}\big].\\\nonumber
\end{align}
Moreover, (\ref{eq: update_theta}) can be realized through distributed estimation to speed up the computations. Firstly, the sample mean $\hat{\mu}$ can be quickly obtained through averaging the local means from each batch of data, that is $\hat{\mu}=k^{-1}\sum_{j=1}^k \hat \mu_j$. And the sample mean is then transferred to each local machine to achieve the centralization $\hat X_i=(X_i\trans-\hat\mu\trans,-1)\trans$. For each batch of data $\calD_j$, we then calculate the following quantities:
\begin{align*}
& \hat U_{j}=n^{-1}\sum_{i \in \calI_j} \hat X_i\hat X_i\trans, \quad
\hat U_{(0),j,\ell}=n^{-1}\sum_{i \in \calI_j} \hat X_i\hat X_i\trans
H^\prime\{g(\hat X_i,\tilde Y_i^{(\ell)},\hat\theta_{(0),\ell})/h\}/h,\\
& \hat V_{(0),j,\ell}=n^{-1}\sum_{i \in \calI_j} \hat X_i\tilde Y^{(\ell)}_i
\{H\{g(\hat X_i,\tilde Y_i^{(\ell)},\hat\theta_{(0),\ell})/h\}+H^\prime\{g(\hat X_i,\tilde Y_i^{(\ell)},\hat\theta_{(0),\ell})/h\}/h\}\big].\\
\end{align*}
$(\hat U_j, \hat U_{(0),j,\ell},\hat V_{(0),j,\ell})$ computed on each local machine are finally aggregated together to fulfill the calculation of (\ref{eq: update_theta}) as
\begin{align}\nonumber
\hat\theta_{(1),\ell}=\big[\sum_{j=1}^k \{ \hat U_{(0),j,\ell} + 2\lambda^{-1}\diag(\hat U_{j},0)\}\big]^{-1} \sum_{j=1}^k  \hat V_{(0),j,\ell}
\end{align}
We then estimate the population level candidate matrix $M_0$ as
\begin{equation}
\label{m-proc}
\widetilde M_{(1)} = \sum_{\ell=1}^{R} \hat\psi_{(1),\ell} \hat\psi\trans_{(1),\ell},
\end{equation}
where $\hat\theta\trans_{(1),\ell}=(\hat\psi\trans_{(1),\ell},\hat t_{(1),\ell})\trans$. And the eigenvectors corresponding to the $d$ largest eigenvalues of $\widetilde M_{(1)}$, denoted by $\widetilde V_{(1)}=(\tilde \nu_{(1),1},\ldots,\tilde \nu_{(1),d})$, are the refined distributed estimators with one step iteration for the central subspace ${\calS}_{Y\mid X}$.

In general, based on the $(B-1)$th iteration estimator $\hat\theta_{(B-1),\ell}$, we can update the parameters through distributed estimation as follows:
\begin{equation}
\label{procedure}
\begin{gathered}
  {{\hat U}_{\left( {B - 1} \right),j,\ell }} = {n^{ - 1}}\sum\limits_{i \in {\calI_j}} {{{\hat X}_i}\hat X_i^TH'\left\{ {g\left( {{{\hat X}_i},\tilde Y_i^{(\ell )},{{\hat \theta }_{(B - 1),\ell }}} \right)/h} \right\}/h}  \hfill \\
  {{\hat V}_{\left( {B - 1} \right),j,\ell }} = {n^{ - 1}}\sum\limits_{i \in {\calI_j}} {{{\hat X}_i}\tilde Y_i^{\left( \ell  \right)}\left\{ {H\left\{ {g\left( {{{\hat X}_i},\tilde Y_i^{(\ell )},{{\hat \theta }_{(B - 1),\ell }}} \right)/h} \right\} + H'\left\{ {g\left( {{{\hat X}_i},\tilde Y_i^{(\ell )},{{\hat \theta }_{(B - 1),\ell }}} \right)/h} \right\}/h} \right\}}  \hfill \\
  {{\hat \theta }_{\left( B \right),\ell }} = {\left[ {\sum\limits_{j = 1}^k {\left\{ {{{\hat U}_{\left( {B - 1} \right),j,\ell }} + 2{\lambda ^{ - 1}}diag\left( {{{\hat U}_j},0} \right)} \right\}} } \right]^{ - 1}}\sum\limits_{j = 1}^k {{{\hat V}_{(B - 1),j,\ell }}}  \hfill \\
\end{gathered}
\end{equation}
And we can further construct the $B$th step candidate matrix $\widetilde M_{(B)}$ and the sufficient dimension reduction estimators $\widetilde{\calV}_{(B)}$ accordingly. The next theorem confirms the asymptotic efficiency of the refined distributed algorithm for the estimation of $\calS_{Y\mid X}$.
\par
The entire procedure is summarized in Algorithm \eqref{alg:rpsvm} in the appendix.

\begin{theorem}\label{theo: asy rpsvm}  Assume the regularity conditions 1-5 and 7 in the Appendix holds true. If $\hat\theta_{(0),\ell}-\theta_{0,\ell}=O_P(m^{-1/2})$ for $1\le \ell \le R-1$, then for $B\ge \lceil C_0 \log_2(\log_m n) \rceil $ with some positive constant $C_0$,
	\begin{align*}
	&n^{1/2}\textup{vec}(\widetilde M_{(B)}-M_{0})\rightarrow N(0_{p^2},\Sigma_{M}),\\
	&n^{1/2}\textup{vec}(\widetilde \calV_{(B)}-\calV_{0})\rightarrow N(0_{pd},\Sigma_{V}),
	\end{align*}
	in distribution as $n\rightarrow\infty$.
\end{theorem}

The refined distributed principal support vector machines through $B$ times iteration is as efficient as the original  support vector machines based on the entire data set.
More importantly, such asymptotic efficiency is attained for a wide range of $m$, which suggests the refined distributed algorithm is advocated when the batch size $m$ is {relatively} small.

As for the initial value $\hat\theta_{(0),\ell}$,  the following proposition suggests that it can be chosen as any $\hat{\theta}_{j,\ell}$ for $1\le j\le k$,  where $\hat{\theta}_{j,\ell}$ is the estimator of ${\theta}_{0,\ell}$ based on the $j$th batch of data. Without loss of generality, we set $\hat\theta_{(0),\ell}=\hat\theta_{1,\ell}$.
\begin{proposition}\label{prop: initial value} Assume the regularity conditions 1-5 in the Appendix, then we have
	\begin{align*}
	\hat\theta_{j,\ell}-\theta_{0,\ell}=O_P(m^{-1/2}),
	\end{align*}
	for $j=1,\ldots,k$ and $\ell=1,\ldots,R-1$.
\end{proposition}

\section{Extensions to Principal Weighted Support Vector Machines }
Like sliced inverse regression, principal support vector machines may work poorly for binary $Y$ when $d>1$, see detail discussions in \cite{BINARYSDR}. To fix this problem, \cite{WPSVM2016} proposed the principal weighted support vector machines for binary $Y=\{-1,+1\}$, which modifies the sample level loss function (\ref{sample:psvm}) as
\begin{align}\label{sample:wpsvm}
\hat L(\psi_{\ell},t_{\ell})=\psi_\ell\trans \hat \Sigma\psi_\ell + \lambda E_n[w_{\pi_{\ell}}(Y)[1-Y\{\psi_{\ell}\trans(X-\hat\mu)-t_{\ell}\} ]^+],
\end{align}
where $w_{\pi_{\ell}}(Y) = 1 - \pi_{\ell}$ if $Y = 1$
and $\pi_{\ell}$
if  $Y=-1$
with a weight $\pi_{\ell} \in (0, 1)$
that controls the relative importance of the two classes for $\ell=1,\ldots,R$. Then, the sample level candidate matrix is
\begin{equation}
\hat M_n^{WL} = \sum\limits_{\ell  = 1}^R {{{\hat \psi }_{n,\ell }}} \hat \psi _{n,\ell }^T,
\end{equation}
where ${\hat \theta _{n,\ell }} = {(\hat \psi _{n,\ell }^T,{\hat t_{n,\ell }})^T}$ are the corresponding minimizer. Similarly, the corresponding estimation of the central subspace $\calV$ can be derived by the first $d$ eigenvectors. And according to \cite{WPSVM2016},  we define
\begin{equation}
\begin{gathered}
  {\tilde D_{{\theta _{0,\ell }}}}({Z^{(\ell )}}) = {(2\psi _{0,\ell }^T\Sigma ,0)^T} - \lambda \left\{ {{w_{{\pi _\ell }}}\left( {{{\tilde Y}^{(\ell )}}} \right)\tilde X{{\tilde Y}^{(\ell )}}I(1 - \theta _{0,\ell }^T\tilde X{{\tilde Y}^{(\ell )}} > 0)} \right\} \hfill \\
  {\tilde H_{{\theta _{0,\ell }}}} = 2diag(\Sigma ,0) + \lambda \sum\limits_{\tilde y = 1, - 1} P ({{\tilde Y}^{(\ell )}} = \tilde y){w_{{\pi _\ell }}}\left( {\tilde y} \right){f_{\psi _{0,\ell }^TX|{{\tilde Y}^{(\ell )}}}}({t_\ell } + \tilde y|\tilde y)E(\tilde X{{\tilde X}^T}|\psi _{0,\ell }^TX = {t_\ell } + \tilde y) \hfill \\
\end{gathered}
\end{equation}
and
\begin{equation}
{\tilde S_{{\theta _{0,\ell }}}}({Z^{(\ell )}}) =  - \tilde H_{{\theta _{0,\ell }}}^{ - 1}{\tilde D_{{\theta _{0,\ell }}}}({Z^{(\ell )}})
\end{equation}
in a similar fashion as with the original PSVM. And the following asymptotic property stands.

\begin{theorem}\label{theo: asy wpsvm}  Let $M_0^{WL} = \sum\limits_{\ell  = 1}^R {{\psi _{0,\ell }}\psi _{0,\ell }^T} $ be the true candidate matrix, then, assume that $\Sigma $ is positive definite and regularity conditions 1-5 in the Appendix holds true. We have
\begin{align*}
	&{n^{1/2}}vec\left( {\hat M_n^{WL} - M_0^{WL}} \right) \to N\left( {{0_{{p^2}}},{\Sigma _M}} \right),\\
	&n^{1/2}\textup{vec}(\hat \calV_n^{WL}-\calV_0^{WL})\rightarrow N(0_{pd},\Sigma_{V}),
	\end{align*}
\end{theorem}

Then for the $j$th batch of data, the naive distributed algorithm for principal weighted support vector machines adopt the following loss functions
\begin{align*}
\hat L_j(\psi_{\ell},t_{\ell})=\psi_\ell\trans \hat \Sigma_j\psi_\ell + \lambda m^{-1}\sum_{i \in \calI_j}[w_{\pi_{\ell}}(Y_i)[1-Y_i\{\psi_{\ell}\trans(X_i-\hat\mu_j)-t_{\ell}\} ]^+].
\end{align*}
In addition,  we denote ${{\hat \theta }_{j,\ell }} = {(\hat \psi _{j,\ell }^T,{{\hat t}_{j,\ell }})^T}$ as the minimizer of ${{\hat L}_j}({\psi _\ell },{t_\ell })$. Then the corresponding sample level candidate matrix can be expressed as
\begin{equation}
{{\tilde M}^{WL}} = \sum\limits_{j = 1}^k {\hat M_j^{WL}} /k,
\end{equation}
where
\begin{equation}
\hat M_j^{WL} = \sum\limits_{\ell  = 1}^R {{{\hat \psi }_{j,\ell }}\hat \psi _{j,\ell }^T}.
\end{equation}

The following theorem confirms the asymptotic efficiency of the naive distributed algorithm of weighted principal support vector machines with sufficiently large $m$.

\begin{theorem}\label{theo: asy nwpsvm} In addition to the regularity conditions 1-6 listed in the Appendix, assume that $m\rightarrow \infty$ and  $k\rightarrow\infty$ such that $n=o(m^{2\gamma})$ where $1/2<\gamma\le 1$ is a positive constant specified in condition 6 in the Appendix. Then we have
	\begin{align*}
	&{n^{1/2}}vec\left( {{{\tilde M}^{WL}} - M_0^{WL}} \right) \to N\left( {{0_{{p^2}}},{\Sigma _M}} \right),\\
	&n^{1/2}\textup{vec}(\widetilde \calV^{WL}-\calV_0^{WL})\rightarrow N(0_{pd},\Sigma_{V}),
	\end{align*}
	in distribution.
	\end{theorem}

For the refined distributed algorithm of principal weighted support vector machines, we consider a smooth version of (\ref{sample:wpsvm})
\begin{align}\label{smooth:wpsvm}
\hat L(\psi_{\ell},t_{\ell})=\psi_\ell\trans \hat \Sigma\psi_\ell + \lambda n^{-1} \sum_{i=1}^n w_{\pi_{\ell}}(Y_i)K_h[1- Y_i\{\psi_{\ell}\trans(X_i-\hat\mu)-t_{\ell}\} ].
\end{align}
Similar to (\ref{eq: smooth_theta}), the optimal $\theta_{\ell}$ that minimizes (\ref{smooth:wpsvm}) should satisfy
\begin{align}\nonumber
\theta_{\ell}=& \big[ \lambda n^{-1}\sum_{i=1}^n  w_{\pi_{\ell}}(Y_i)\hat X_i\hat X_i\trans
H^\prime\{g(X_i, Y_i,\theta_{\ell})/h\}/h+2\diag(\hat\Sigma,0)\big]^{-1}
\\\label{update:theta1}
&\big[\lambda n^{-1}\sum_{i=1}^n  w_{\pi_{\ell}}(Y_i)\hat X_i Y_i
\{H\{g(X_i,Y_i,\theta_{\ell})/h\}+H^\prime\{g(\hat X_i, Y_i,\theta_{\ell})/h\}/h\}\big].
\end{align}

Parallel to the developments in the previous section, we will solve the optimization problem (\ref{sample:wpsvm}) through distributed estimation and recursive programming. Given the $(B-1)$th step estimator $\tilde\theta_{(B-1),\ell}$ $(B\geq 1)$ we calculate the following quantities based on the $j$th batch of data:
\begin{align*}
& \tilde U_{(B-1),j,\ell}=n^{-1}\sum_{i \in \calI_j} w_{\pi_{\ell}}(Y_i)\hat X_i\hat X_i\trans
H^\prime\{g(\hat X_i, Y_i,\tilde\theta_{(B-1),\ell})/h\}/h,\\
& \tilde V_{(B-1),j,\ell}=n^{-1}\sum_{i \in \calI_j} w_{\pi_{\ell}}(Y_i)\hat X_i\tilde Y^{(\ell)}_i
\{H\{g(\hat X_i,Y_i,\tilde\theta_{(B-1),\ell})/h\}+H^\prime\{g(\hat X_i, Y_i,\tilde\theta_{(B-1),\ell})/h\}/h\}\big].
\end{align*}
In view of (\ref{update:theta1}),  we then update the estimation as
\begin{align*}
\tilde\theta_{(B),\ell}=\big[\sum_{j=1}^k \{ \tilde U_{(B-1),j,\ell} + 2\lambda^{-1}\diag(\hat U_{j},0)\}\big]^{-1} \sum_{j=1}^k  \tilde V_{(B),j,\ell}.
\end{align*}

And we can further construct the candidate matrix and utilize the top $d$ eigenvectors to estimate the central subspace $\calS_{Y\mid X}$. Similarly, the candidate matrix will be
\begin{equation}
{{\tilde M}_{\left( B \right)}} = \sum\limits_{\ell  = 1}^R {{{\tilde \psi }_{\left( B \right),\ell }}\tilde \psi _{\left( B \right),\ell }^T} ,
\end{equation}
where ${{\tilde \theta }_{\left( B \right),\ell }} = {\left( {{{\tilde \psi }_{\left( B \right),\ell }},{{\tilde t}_{\left( B \right),\ell }}} \right)^T}$.

\par
Moreover, along with the theoretical investigations in Theorem \ref{theo: asy rpsvm}, we can also establish the asymptotic efficiency results for the naive and refined distributed estimators of principal weighted support vector machines.

\begin{theorem}\label{theo: asy rwpsvm}  Assume the regularity conditions 1-5 and 7 in the Appendix holds true. If $\widetilde\theta_{(0),\ell}-\theta_{0,\ell}=O_P(m^{-1/2})$ for $1\le \ell \le R-1$, then for $B\ge \lceil C_0 \log_2(\log_m n) \rceil $ with some positive constant $C_0$,
	\begin{align*}
	&n^{1/2}\textup{vec}(\widetilde M_{(B)}-M_{0})\rightarrow N(0_{p^2},\Sigma_{M}),\\
	&n^{1/2}\textup{vec}(\widetilde \calV_{(B)}-\calV_{0})\rightarrow N(0_{pd},\Sigma_{V}),
	\end{align*}
	in distribution as $n\rightarrow\infty$.
\end{theorem}

\section{Simulation Studies}
In this section, we conduct extensive monte carlo simulations to examine our proposed methods. Our simulation studies include 36 different combinations of $(n,p,k)\in \{30000,50000,100000\}\times \{10,20,30\}\times \{10,50,100,500\}$.  We generate data from the following four models:

\begin{eqnarray*}
	\text{I:} \quad Y&=&x_1/(0.5+(x_2+1)^2)+\varepsilon\\
	\text{II:} \quad Y&=&x_1(x_1+x_2+1)+\varepsilon\\
	\text{III:} \quad Y&=&\text{sign}(x_1/(0.5+(x_2+1)^2)+\varepsilon)\\
	\text{IV:} \quad Y&=&\text{sign}(x_1(x_1+x_2+1)+\varepsilon)
\end{eqnarray*}
where $X=(x_1,\ldots,x_p)\trans \thicksim N(0_p,I_p)$ and the error $\varepsilon \thicksim N(0,0.5^2)$. Model I and II with continuous response are used in \cite{PSVM2011} to demonstrate the effectiveness of principal support vector machines (PSVM). Model III and IV with binary response which favor principal weighted support vector machines (WPSVM) are adopted in \cite{WPSVM2016}.

For principal support vector machines, the number of slices is set as $R=5$.  And for the weighted support vector machines, we also use $R=5$  values equally spaced in $[0,1]$ as the weights $\pi_\ell$'s. According to the theoretical findings in \cite{ASY1SVM2008} and \cite{ASY2SVM2008},  $\lambda$ is chosen as $2n^{2/3}$ for the  principal (weighted) support vector machines and is chosen as  $2m^{2/3}$ for the distributed algorithms. Similar to \cite{DQR2018}, the bandwidth $h$ is chosen as $ \max\{10(p/n)^{1/2},10(p/m)^{2^{B-2}},0.3\}$ for the $B$th step iteration in the refined distributed algorithm. And the total number of iterations in our numerical studies is set as $B=3$.  In addition,  we adopt the following smoothing function for the refined distributed algorithm for the refined distributed algorithm:
$$ H(v)=\left\{
\begin{array}{rcl}
&0      ,&\ \  v\in(-\infty,-1],\\
&\frac{1}{2}+\frac{15}{16}(v-\frac{2}{3}v^3+\frac{1}{5}v^5)      ,&\ \  v\in[-1,1],\\
&1     , &\ \  v\in[1,\infty).
\end{array} \right. $$

We first compare the accuracy and the computational cost of each method with a relatively small sample size $n=30000$. For the above four models, the structural dimension $d$ are all equal to $2$, and $\text{Span}(\calS_{Y\mid X})=\text{Span}(\calV)=(e_1,e_2)$. To assess the the performance of each estimator $\widehat{\calV}$, we adopt the distance measure $d(\widehat{\calV}, \calV)=\|P_{\widehat{\calV}}-P_{
\calV}\|_{F}$, where $P_{\calV}$ and $P_{\widehat{\calV}}$ are orthogonal projections on to $\calV$ and $\widehat{\calV}$, and $\|.\|_F$ stands for the Frobenius norm. Table 1 summarizes the mean of distances calculated from $200$ simulated samples for $n=30000$. In Table 2, we report the average running time for $n=30000$. As our computing resource is limited, the naive and refined distributed algorithms are actually implemented on a single machine with the computation time recorded as if in a parallel setting.
From Table 1, we observe that the refined distributed estimation of PSVM (RD-PSVM) and refined distributed estimation of WPSVM (RD-WPSVM) performs better than the naive distributed estimation of PSVM (ND-PSVM) and naive distributed estimation of WPSVM (ND-WPSVM) separately, especially when $k$ is getting larger. The accuracy of the naive algorithm drops considerably with $(m,k)=(60,500)$, which is consistent with the theoretical limitation of the naive approach explored in Theorem \ref{theo: asy npsvm}. The refined distributed estimator is quite robust to the choice of $k$ as expected from Theorem \ref{theo: asy rpsvm}. The refined distributed estimator with moderate $m$ is comparable to the standard principal (weighted) support vector machines in estimating $\calS_{Y\mid X}$. Moreover, compared to the original estimation, the refined distributed algorithm reduces the computational burden significantly, {which can be verified in} Table 2.

We in the next focus on the comparison of the two distributed algorithms with large $n$. Table 3 and 4 report $d(\widehat{\calV}, \calV)$
averaged over $200$ repetitions for $n = 50000$ and $n=100000$.
The original principal (weighted) support vector machine estimators are not included for comparison as the implementation is very {time-consuming}.
The naive distributed algorithm again tends to deteriorate when $k$ becomes larger. However, the mean distances are getting smaller as $n$ increases, which echoes the large sample results.  The two distributed algorithms are thus highly recommended for sufficient dimension  {reduction} with massive datasets, as they take into account both statistical accuracy and computational complexity.

\begin{table}[ht]\label{t1}
	\def~{\hphantom{0}}
	\caption{\it Average distances in estimating the central subspace with $n=30000$. }{%
		\begin{tabular}{cccccccccccc}
			\\	
			\multicolumn{6}{c}{Model I}&\multicolumn{6}{c}{Model II}\\[6pt]
			$p=10$	&	500      &   100	  &	  50	     & 	10	    &	 $p=10$	&	 500     &   100	  &	  50	    &	10	    \\
			ND-PSVM         &	0.257	 &   0.163    &   0.172   	 &  0.163	&	   ND-PSVM   	&	0.189	 &   0.106    &   0.106   	&  0.079	\\
			RD-PSVM         &	0.196	 &   0.150    &   0.147	     &  0.146	&	   RD-PSVM	    &	0.108	 &   0.091    &   0.088	    &  0.078	\\
			PSVM	        &	0.076    	 &   -    &     -    	 &   -   	&	   PSVM	    &	  0.076       & 	 -    &      -       &	    -	\\
			[6pt]																		
			$p=20$	&	500      &   100	  &	  50	     &	10	    &	 $p=20$	&	 500     &   100	  &	  50	   &	10	    \\
			ND-PSVM         &	0.554	 &   0.288    &   0.254   	 &  0.255	&	   ND-PSVM	    &	 0.381   &	 0.159	  &   0.148	   &	0.132	\\
			RD-PSVM         &	0.301	 &   0.317    &   0.225	     &  0.214	&	   RD-PSVM  	&	 0.147   &	 0.151	  &	  0.128	   &	0.127	\\
			PSVM	        &	  0.205       &   -    &     -     	 &    -  	&	   PSVM	    &	  0.116       &	 -	  &  -          &	   - 	\\
			[6pt]																		
			$p=30$	&   500      &   100	  &   50	     &	10	    &	 $p=30$	&	 500     &   100	  &	  50	   &	10	    \\
			ND-PSVM	        &	1.137	 &   0.332    &   0.326      &  0.259	&	   ND-PSVM	    &	 0.658   & 	 0.237	  &   0.199	   &	0.164	\\
			RD-PSVM         &	0.611	 &   0.402    &   0.297	     &  0.256	&	   RD-PSVM  	&	 0.301   &	 0.159    &	  0.180    &	0.160	\\
			PSVM      	&	0.249    	 &   -    &     -         &    -    	&	   PSVM	    &	 0.151        &	 -    &     - 	   &	-        \\
			[8pt]																		
			\multicolumn{6}{c}{Model III}&\multicolumn{6}{c}{Model IV}\\[6pt]
			$p=10$	&	500      &   100	  &   50	     &	10	    &	 $p=10$	&	 500     &   100	  &	  50	   &	10	    \\
			ND-WPSVM         &	0.656	 &   0.229    &   0.215   	 &  0.130	&	   ND-WPSVM	    &	 0.452   &	 0.139	  &   0.137	   &	0.101	\\
			RD-WPSVM         &	0.135	 &   0.128    &   0.118	     &  0.117	&	   RD-WPSVM	    &	 0.086   &   0.084	  &	  0.075	   &	0.073	\\
			WPSVM	        &	  0.085   	 &   -    &     -     	 &     - 	&	   WPSVM	    &	 0.089        &	 -    &   -         &	-     	\\
			[6pt]																		
			$p=20$	&	500      &   100	  &	  50	     &	10	    &	$p=20$	&	 500     &   100	  &	  50	   &	10	    \\
			ND-WPSVM         &	1.195	 &   0.373    &   0.328   	 &  0.183	&	   ND-WPSVM	    &	 1.242   &	 0.247	  &   0.155	   &	0.141	\\
			RD-WPSVM         &	0.195	 &   0.203    &   0.201	     &  0.173	&	   RD-WPSVM	    &	 0.164   &	 0.150	  &	  0.143	   &	0.123	\\
			WSPVM	        &	0.157     	 &   -    &      -        &   -    	&	   WPSVM	    &	 0.134        &	 -   &   -   	   & -	        \\
			[6pt]																		
			$p=30$	&	500      &   100	  &	  50	     &	10	    &	 $p=30$	&	 500     &   100	  &	  50	   &	10   	\\
			ND-WPSVM        &	1.342	 &   0.522    &   0.360   	 &  0.208	&	   ND-WPSVM   	&	 1.382   &	 0.313	  &   0.217	   &	0.191	\\
			RD-WPSVM	        &	0.475	 &   0.258    &   0.244	     &  0.228	&	   RD-WPSVM   	&	 0.238   &	 0.209	  &	  0.231	   &	0.152	\\
			WPSVM        & 0.181	     	 &   -    &     -     	 &    -      &	   WPSVM	    &	  0.167       &	 -    &     -       &	-    	\\
			[6pt]
	\end{tabular}}
\end{table}

\begin{table}[ht]\label{t2}
	\def~{\hphantom{0}}
	\caption{\it Average running time (in seconds) with $n=30000$. }{%
		\begin{tabular}{cccccccccccc}
			\\	
			\multicolumn{6}{c}{Model I}&\multicolumn{6}{c}{Model II}\\[6pt]
			$p=10$	&	k=500      &   k=100	  &	  k=50	     & 	k=10	    &	 $p=10$	&	 k=500     &   k=100	  &	  k=50	    &	k=10	    \\
			ND-PSVM         &	0.004	 &   0.011    &   0.028	     &  0.487	&	   ND-PSVM      &	 0.006   &   0.010	  &	  0.026	    &	0.228	\\
			RD-PSVM	        &	0.133	 &   0.242    &   0.386 	 &  1.848	&	   RD-PSVM     	&	 0.106   &	 0.218	  &   0.367 	&	1.454	\\
			PSVM        &	33.962    	 &   -   &      -	     &    -  	&	   PSVM     	&	10.702         &	 -   &	     -       &	    -	\\
			[6pt]																		
			$p=20$	&	k=500      &   k=100	  &	  k=50	     &	k=10	    &	 $p=20$	&	 k=500     &   k=100	  &	  k=50	   &	k=10	    \\
			ND-PSVM         &	0.005	 &   0.017    &   0.041	     &  0.702	&	   ND-PSVM     	&	 0.005   &	 0.014	  &	  0.030	   &	0.291	\\
			RD-PSVM	        &	0.173	 &   0.291    &   0.421	     &  2.070	&	   RD-PSVM     	&	 0.130   &   0.246	  &	  0.362	   &	1.585	\\
			PSVM        &	 51.707        &   -   &      -	     &     - 	&	   PSVM      &	  15.703       &	 -   &	   -        &	   - 	\\
			[6pt]
			$p=30$																		
			&   k=500      &   k=100	  &   k=50	     &	k=10	    &	 $p=30$	&	 k=500     &   k=100	  &	  k=50	   &	k=10	    \\
			ND-PSVM         &	0.016	 &   0.027    &   0.062	     &  0.982	&	   ND-PSVM      &	 0.007   &	 0.019	  &	  0.041	   &	0.374	\\
			RD-PSVM	        &	0.212	 &   0.370    &   0.493	     &  2.445	&	   RD-PSVM      &	 0.158   &	 0.256	  &	  0.406	   &	1.689	\\
			PSVM	        &	 79.971   	 &   -   &      -	     &     - 	&	   PSVM      &	   23.846      &	 -   &	   -  	   &	 -       \\
			[8pt]																		
			\multicolumn{6}{c}{Model III}&\multicolumn{6}{c}{Model IV}\\[6pt]
			$p=10$	&	k=500      &   k=100	  &   k=50	     &	k=10	    &	 $p=10$	&	 k=500     &   k=100	  &	  k=50	   &	k=10	    \\
			ND-PSVM         &	0.004	 &   0.020    &   0.045	     &  0.653	&	   ND-PSVM      &	 0.006   &	 0.016	  &	  0.041	   &	0.715	\\
			RD-WPSVM	        &	0.289	 &   0.346    &   0.459	     &  1.433	& 	   RD-WPSVM      &	 0.283   &	 0.358	  &	  0.417    &	1.504	\\
			WPSVM	        &	  49.845  	 &   -   &      -	     &     - 	&      WPSVM     	&	   53.106      &	 -   &	     -      &	    -  	\\
			[6pt]																		
			$p=20$	&	k=500      &   k=100	  &	  k=50	     &	k=10	    &	$p=20$	&	 k=500     &   k=100	  &	  k=50	   &	k=10	    \\
			ND-WPSVM         &	0.005	 &   0.021    &   0.055	     &  0.887	&	   ND-WPSVM      &	 0.005   &	 0.020	  &	  0.054	   &	0.919	\\
			RD-WPSVM        &	0.302	 &   0.373    &   0.419	     &  1.636	&	   RD-WPSVM      &	 0.304   &   0.357	  &	  0.434    &	1.696	\\
			WPSVM	        &	  74.560   	 &   -   &       -       &      - 	&	   WPSVM      &	    80.322     &	 -   &	    - 	   &	   -     \\
			[6pt]																		
			$p=30$	&	k=500      &   k=100	  &	  k=50	     &	k=10	    &	 $p=30$	&	 k=500     &   k=100	  &	  k=50	   &	k=10   	\\
			ND-WPSVM         &	0.008	 &   0.028    &   0.072	     &  1.197	&      ND-WPSVM      &	 0.007   &	 0.028	  &	  0.072	   &	1.187	\\
			RD-WPSVM	        &	0.301	 &   0.381    &   0.463	     &  1.944	&	   RD-WPSVM     	&	 0.304   &	 0.392	  &   0.460	   &	2.006	\\
			WPSVM	        &	  114.767       &  - &      -	     &    -      &	  WPSVM     	&	 122.767        &	 -  &	  -         &	 -    	\\
			[6pt]
	\end{tabular}}
\end{table}

\begin{table}\label{ta3}
	\def~{\hphantom{0}}
	\caption{\it Average distances in estimating the central subspace with $n=50000$. }{%
		\begin{tabular}{cccccccccccc}
			\\	
			\multicolumn{6}{c}{Model I}&\multicolumn{6}{c}{Model II}\\[6pt]
			$p=10$    	    &	k=500     &  k=100	   &   k=50	    &	k=10	    &	$p=10$     	    &	k=500     &   k=100	   &	k=50	    &	k=10	    \\
			ND-PSVM	                &   0.213	&  0.147   &   0.122	&	0.109	&	 ND-PSVM               &   0.193	&  0.137   &   0.081	&	0.069	\\
			RD-PSVM                 &   0.181 	&  0.104   &   0.106	&	0.093	&	 RD-PSVM                &   0.177 	&  0.100   &   0.080	&	0.067	\\
			[6pt]																		
			$p=20$          &	k=500     &  k=100	   &   k=50	    &	k=10	    &	$p=10$           &	k=500     &   k=100	   &	k=50	    &	k=10	    \\
			ND-PSVM	                &	0.319	&  0.209   &   0.191	&	0.166	&	 ND-PSVM	            &	0.207	&  0.138   &   0.103	&	0.121	\\
			RD-PSVM	                &	0.287 	&  0.174   &   0.181	&	0.138	&	 RD-PSVM                &	0.127 	&  0.117   &   0.097	&	0.122	\\
			[6pt]																		
			$p=30$       	&	k=500     &  k=100	   &   k=50	    &	k=10	    &	$p=30$       	&	k=500     &   k=100	   &	k=50	    &	k=10	    \\
			ND-PSVM	                &	0.427	&  0.255   &   0.247	&   0.209	&	 ND-PSVM	            &	0.304	&  0.148   &   0.149	&   0.129	\\
			RD-PSVM	                &	0.281 	&  0.274   &   0.210	&   0.206	&	 RD-PSVM    	        &	0.145 	&  0.127   &   0.127	&   0.119	\\
			[8pt]																		
			\multicolumn{6}{c}{Model III}&\multicolumn{6}{c}{Model IV}\\[6pt]
			$p=10$	&	k=500     &  k=100	   &   k=50	    &	k=10	    &	$p=10$	&	k=500    &   k=100	    &	k=50	    &	k=10	    \\
			ND-WPSVM	        &	0.294	&  0.191   &   0.154	&	0.097	&	 ND-WPSVM        &	0.218  &   0.110	&   0.085	&	0.060	\\
			RD-WPSVM         &	0.104	&  0.115   &   0.088	&	0.094	&	 RD-WPSVM    	&	0.062  &   0.121	&   0.063	&	0.061	\\
			[6pt]																		
			$p=20$	&	k=500     &  k=100	   &   k=50	    &	k=10	    &	$p=20$	&	k=500    &   k=100	    &	k=50	    &	k=10	    \\
			ND-WPSVM	        &	0.657	&  0.224   &   0.220	&	0.145	&	 ND-WPSVM        &	0.507  &   0.137	&   0.135	&	0.122	\\
			RD-WPSVM	        &	0.182 	&  0.179   &   0.147	&	0.141	&	 RD-WPSVM	    &	0.142  &   0.174	&   0.115	&	0.103	\\
			[6pt]																		
			$p=30$	&	k=500     &  k=100	   &   k=50	    &	k=10	    &	$p=30$	&	k=500    &   k=100	    &	k=50	    &	k=10   	\\
			ND-WPSVM	        &	1.106	&  0.301   &   0.223	&	0.164	&	 ND-WPSVM       &	1.096  &   0.649	&   0.150	&	0.156	\\
			RD-WPSVM	        &	0.255	&  0.259   &   0.191	&	0.170	&    RD-WPSVM	    &	0.286  &   0.217	&   0.196	&	0.127	\\
			[6pt]
	\end{tabular}}
\end{table}

\begin{table}\label{ta4}
	\def~{\hphantom{0}}
	\caption{\it Average distances in estimating the central subspace with $n=100000$. }{%
		\begin{tabular}{cccccccccccc}
			\\	
			\multicolumn{6}{c}{Model I}&\multicolumn{6}{c}{Model II}\\[6pt]
			$p=10$	&	k=500     &  k=100	   &   k=50	    &	k=10	    &	$p=10$	&	k=500    &   k=100	    &	k=50	    &	k=10	    \\
			ND-PSVM	        &   0.135	&  0.097   &   0.085	&	0.083	&	 ND-PSVM        &	0.061  &   0.052	&   0.048	&   0.053	\\
			RD-PSVM         &   0.112 	&  0.092   &   0.201	&	0.066	&	 RD-PSVM        &	0.055  &   0.040	&   0.048	&   0.052	\\
			[6pt]																		
			$p=20$	&	k=500     &  k=100	   &   k=50	    &	k=10	    &	$p=20$	&	k=500    &   k=100	    &	k=50	    &	k=10	    \\
			ND-PSVM	        &	0.170	&  0.155   &   0.137	&	0.135	&	 ND-PSVM	    &	0.114  &   0.093	&   0.075	&   0.077	\\
			RD-PSVM	        &	0.183 	&  0.128   &   0.139	&	0.098	&	 RD-PSVM        &	0.076  &   0.078	&   0.067	&   0.075	\\
			[6pt]																		
			$p=30$	&	500     &  100	   &   50	    &	10	    &	$p=30$	&	500    &   100	    &	50	    &	10	    \\
			ND-PSVM	        &	0.210	&  0.159   &   0.157	&   0.142	&	 ND-PSVM	    &	0.153  &   0.102	&   0.093	&	0.097	\\
			RD-PSVM	        &	0.181 	&  0.158   &   0.156	&   0.147	&	 RD-PSVM    	&	0.101  &   0.099	&   0.095	&	0.087	\\
			[8pt]																		
			\multicolumn{6}{c}{Model III}&\multicolumn{6}{c}{Model IV}\\[6pt]
			$p=10$	&	k=500     &  k=100	   &   k=50	    &	k=10	    &	$p=10$	&	k=500    &   k=100	    &  k=50	    &	k=10	    \\
			ND-WPSVM	        &	0.103	&  0.064   &   0.058	&	0.060	&	 ND-WPSVM        &	0.188  &   0.088	&   0.071	&	0.067	\\
			RD-WPSVM         &	0.068	&  0.055   &   0.045	&	0.047	&	 RD-WPSVM   	&	0.070  &   0.066	&   0.065	&	0.075	\\
			[6pt]																		
			$p=20$	&	k=500     &  k=100	   &   k=50	    &	k=10	    &	$p=20$	&	k=500    &   k=100	    &	k=50	    &	k=10	    \\
			ND-WPSVM	        &	0.279	&  0.137   &   0.106	&	0.091	&	 ND-WPSVM        &	0.167  &   0.094	&   0.099	&	0.078	\\
			MD-WPSVM	        &	0.147	&  0.114   &   0.105	&	0.102	&	 RD-WPSVM	    &	0.117  &   0.090	&   0.075	&	0.070	\\
			[6pt]																		
			$p=30$	&	k=500     &  k=100	   &   k=50	    &	k=10	    &	$p=30$	& k=500    & k=100	    &  k=50	    &	k=10   	\\
			ND-WPSVM	        &	0.377	&  0.168   &   0.126	&	0.112	&	 ND-WPSVM        &	0.239  &   0.109	&   0.116	&	0.098	\\
			RD-WPSVM        &	0.159	&  0.155   &   0.122	&	0.121	&    RD-WPSVM	    &	0.181  &   0.120	&   0.097	&	0.084	\\
			[6pt]
	\end{tabular}}
\end{table}

\section{Boeing 737 Data Analysis}
We now compare the principal support vector machines method with the proposed distributed algorithms in a real data analysis. The data contains $14$ index variables of $618178$ flights conducted by Boeing 737 throughout the landing process. The $14$ measured values during the landing procedure, include the maximal pitch angle, the maximal airspeed, the average airspeed, the maximal groundspeed, the total ground distance, the total elapsed time, the difference in fuel consumed engine $1$, the average fuel flow engine $1$, the maximal { Mach (a unit of speed)}, the average  {Mach}, the maximal absolute longitudinal acceleration, the maximal absolute lateral acceleration, the minimal  vertical acceleration, the maximal vertical {acceleration}. In a landing action, if the plane lands too fast, a huge vertical acceleration will be generated and accordingly a considerable gravitational force will be acted on the landing gear, jeopardizing the quality and safety of a flight. Therefore, we adopt the maximal vertical  {acceleration} during landing as the response $Y$ and the rest $13$ {indices} as the explanatory variables.

We first apply the original principal support vector machines method to this data for the estimation of the $\calS_{Y|X}$. For this data with $n=618178$, we set the number of slices as $R=10$. And the computation time for this massive data is $90028.17$ seconds. The top $3$ eigenvalues are $2484.7$, $69.2$, $0.5$ respectively, and the rest eigenvalues are all smaller than $0.02$. The ridge ratio based BIC-type method proposed by \cite{RidgeBIC} further yields $\hat d=2$  as the estimation of the structural dimension. Denote the resulting estimator of principal support vector machine by $\widehat\calV_n$, which is a $13\times 2$ matrix. And the distance correlation \citep{Dcor2007} between $Y$ and $\widehat\calV_n\trans X$, represented by $\text{dcor}(Y,\widehat\calV_n\trans X)$ is $0.2848$.

We also apply the two distributed algorithms of principal support vector machines to this data for comparisons. The estimator of $\calS_{Y|X}$ based on the distributed algorithms is denoted as $\widetilde\calV$. We calculate the $d({\widehat\calV_n}, \widetilde\calV)$, which {measures} the distance between the distributed estimator and the original estimator. The distance correlation between $Y$ and $\widetilde\calV\trans X$ is also included in our calculations. We summarize all these results along with the computation time in Table 5. It is clear that the distributed algorithms work much faster than the original principal support vector machines method. And the refined distributed estimator is generally very close to the  {original} estimator and is insensitive to the choice of $k$, which again supports our theoretical findings. Although the naive distributed estimator is not very close to the original estimator when $k$ is large, the corresponding distance correlation is very close to the oracle value $0.2848$, which implies the estimated directions still capture useful information for the regression.

\begin{table}[ht]
	\def~{\hphantom{0}}
	\caption{\it Results for the distributed algorithm applied to the Boeing 737 Data.}{%
		\begin{tabular}{lccccc}
		&	& $k=2000$ & $k=1000$ & $k=500$ & $k=100$  \\
$d({\widehat\calV_n}, \widetilde\calV) $	& ND-PSVM & $1.4092$ & $1.2913$ & $0.5025$ & $0.1234$  \\
	& RD-PSVM & $0.1308$ & $0.1293$ & $0.0980$ & $0.1035$  \\
$\text{dcor}(Y,\widetilde{\calV}\trans X)$	& ND-PSVM &$0.1928$ & $0.2784$ & $0.2859$ & $0.2855$  \\
	& RD-PSVM &$0.2840$ & $0.2840$ & $0.2843$ & $0.2841$  \\
	running time	& ND-PSVM &$0.0511$ & $0.2109$ & $0.5950$ & $14.0733$  \\
	& RD-PSVM &$2.1529$ & $3.5482$ & $4.9056$ & $18.9438$  \\
	\end{tabular}}
\end{table}

Finally, based on the sufficient dimension reduction estimators obtained, we can create the 3D scatter plot (Figure 2 and Figure 3) to scrutinize that whether the two distributed algorithms generate a close estimator to the original method or not. In this plot, the $x$ and $y$ axises, i.e. the axises on bottom plane, are ${X^T}{\hat v_1}$ and ${X^T}{\hat v_2}$, where ${\hat v_1}$ and ${\hat v_2}$ are the first two estimated directions. And the $z$ axis characterizes the response value. Moreover, in these figures, the circle points represent the feature extraction of the original method, while the  asterisks and squares represent the naive and refined distributed estimators respectively. It is clear that the two distributed methods can well capture the key regression patterns. In Figure 2 with $k=100$, the extracted features from different methods are almost in the same spatial position, which indicates that the two distributed estimators are close enough to the original principal support vector machines estimator.

\par
On the application front, the return implies that "Max Mach during Landing" (speed measurement) and “Average Mach during Landing” (speed measurement) are two most significant influencing factors to the vertical acceleration on landing moment, which is highly recognized by the aviation industry. It can at least show the effectiveness of our refined method to some extent from another side.

\begin{figure}
  \centering
  \includegraphics[width=115mm]{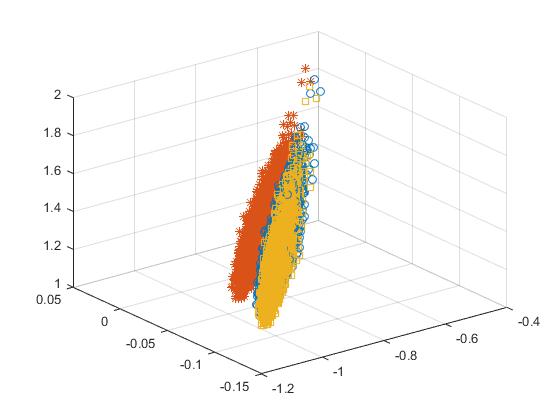}
  \caption{3D scatter plot of feature extraction with $k=100$.}
  \label{fig:boat1}
\end{figure}

\begin{figure}
\centering     
\subfigure[$k=500$]{\label{fig:a}\includegraphics[width=45mm]{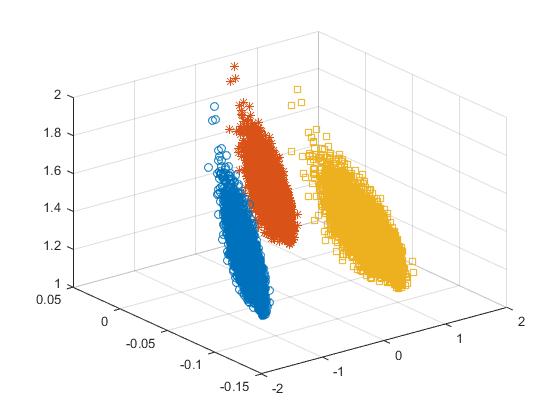}}
\subfigure[$k=1000$]{\label{fig:b}\includegraphics[width=45mm]{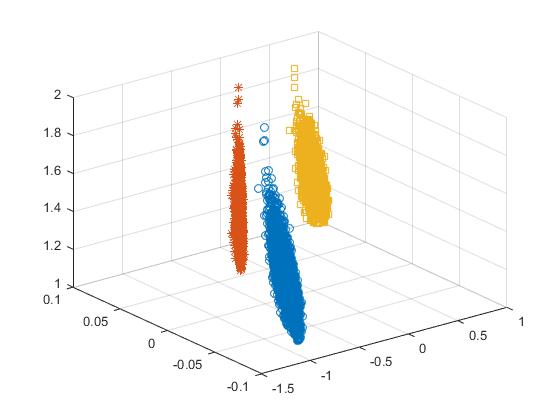}}
\subfigure[$k=2000$]{\label{fig:b}\includegraphics[width=45mm]{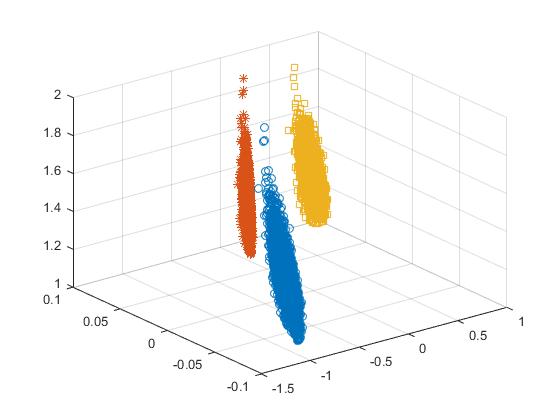}}
\caption{3D scatter plot of feature extraction with different $k$'s.}
\end{figure}

\section{Appendix}

\subsection{Algorithm Map}

\begin{center}
  \begin{minipage}{\linewidth} 
    \begin{algorithm}[H]
      \KwIn{Samples stored in the machines $S = \left\{ {{H_1}, \ldots ,{H_k}} \right\}$; $R-1$ dividing points; Smoothing function $H$; Regularization parameter $\lambda$; Number of iterations $T$; Bandwidth $\left\{ {{h_1}, \ldots ,{h_T}} \right\}$.}  
      \KwResult{$\left( {{{\tilde M}_{\left( T \right)}},{{\tilde V}_{\left( T \right)}}} \right)$} 
      \medskip
      Compute $\hat \Sigma  = Var\left( X \right)$. \;
      \For{${q_\ell } \in \left\{ {{q_1}, \ldots ,{q_{R - 1}}} \right\}$}{
        \For{$g = 1, \ldots ,T$}{
          \If{$g=1$}{
            Compute the initiator based on $H_1$
            \[{{\hat \theta }_{\left( 0 \right),\ell }} \in \mathop {\arg \min }\limits_\theta  {\theta ^T}diag\left( {\hat \Sigma ,0} \right)\theta  + \lambda {n^{ - 1}}\sum\limits_{i \in {\calI_1}} {{K_h}\left( {g\left\{ {{{\hat X}_i},\tilde Y_i^{\left( \ell  \right)},\theta } \right\}/h} \right)} \]
          }
          ${{\hat \theta }_{\left( 0 \right),\ell }}$ is assigned into all the machines. \;
          \For{$k = 1, \ldots K$}{
            Compute $\left( {{{\hat U}_{\left( g \right),k,\ell }},{{\hat V}_{\left( g \right),k,\ell }}} \right)$ according to equation \eqref{procedure}. \;
            Transform all $\left( {{{\hat U}_{\left( g \right),k,\ell }},{{\hat V}_{\left( g \right),k,\ell }}} \right)$ into the central machine. \;
          }
        Compute ${{\hat \theta }_{\left( T \right),\ell }}$ according to equation \eqref{procedure}. \;
        }
      Compute $\left( {{{\tilde M}_{\left( T \right)}},{{\tilde V}_{\left( T \right)}}} \right)$ according to the equation \eqref{m-proc}. \;

      }
      {\bf return} $\left( {{{\tilde M}_{\left( T \right)}},{{\tilde V}_{\left( T \right)}}} \right)$ \;
      \caption{\texttt{Refined distributed estimation}} 
      \label{alg:rpsvm}   
    \end{algorithm}
  \end{minipage}
\end{center}

\subsection{Regularity Conditions}
The following regularity assumptions are necessary for the theoretical investigations.\\ \\
\textit{Assumption} 1. $X$ has an open and  {convex} support and $E(\|X\|^2)<\infty$.\\ \\
\textit{Assumption} 2. The condition  {distribution} of $X|\tilde Y^{(\ell)}$ is dominated by the Lebesgue measure. \\ \\
\textit{Assumption} 3. For any linear independent $\psi, \delta \in \mathbb{R}^p $, $\tilde y=1, -1$ and $v\in \mathbb{R}$, the mapping function $u \mapsto E( X |\psi\trans X =u, \delta\trans X= v, \tilde Y^{(\ell)}=\tilde y) f_{\psi\trans X|\delta\trans X, \tilde Y^{(\ell)} } (u,v|\tilde y) $ is continuous, where  $f_{\psi\trans X|\delta\trans X, \tilde Y^{(\ell)} } $ is the conditional density function of $\psi\trans X$ given $\delta\trans X$ and $ \tilde Y^{(\ell)}$.\\ \\
\textit{Assumption} 4. There exists a  {nonnegative} $\mathbb{R}^{p+1}$-function $C(v,\tilde y)$ with $E\{C(\delta\trans X,\tilde Y^{(\ell)})|\tilde Y^{(\ell)}=\tilde y\} \le \infty$ and $vE( X |\psi\trans X =u, \delta\trans X= v, \tilde Y^{(\ell)}=\tilde y) f_{\psi\trans X|\delta\trans X, \tilde Y^{(\ell)} } (u,v|\tilde y) \le C(v,\tilde y)$, where the  {inequality} holds componentwise. \\ \\
\textit{Assumption} 5. There exists a {nonnegative} function $c_0(v,\tilde y)$ with $E\{c_0(\delta\trans X,\tilde Y^{(\ell)})|\tilde Y^{(\ell)}=\tilde y\} \le \infty$ and $ f_{\psi\trans X|\delta\trans X, \tilde Y^{(\ell)} } (u,v|\tilde y) \le c_0(v,\tilde y)$, where the  {inequality} holds componentwise. \\ \\
\textit{Assumption} 6. There exists a  {nonnegative} constant $\gamma$ such that $1/2<\gamma\leq 1$ $E\hat\theta_{j,\ell}-\theta_{j,\ell}=O_P(m^{-\gamma})$. \\ \\
\textit{Assumption} 7. Assume the bandwidth $h$ satisfies that $h\rightarrow 0$ and $\log n/nh=o(1)$. In addition, for the $b$th iteration, the bandwidth is chosen as $h:=h_b=\max \{ {{n^{ - 1/2}},{m^{ - {2^{b - 2}}}}} \}$ for $1 \leqslant b \leqslant B$. \\ \\

\begin{remark}
Assumptions 1-5 are all utilized in \cite{PSVM2011} and \cite{WPSVM2016} to study the asymptotic  {behavior} of (weighted) principal support vector machines. These are common regularity conditions for the  {asymptotic} analysis of support vector machine related problems. Assumption 6 can be regarded as a conclusion from Theorem 6 in \cite{PSVM2011}, which asserts that
\begin{align*}
\hat\theta_{j,\ell}=&\theta_{j,\ell}+m^{-1} \sum_{i\in \calI_j} S_{\theta_{0,\ell}}(Z_i^{(\ell)}) +o_P(m^{-1/2}).
\end{align*}
As $E\{ S_{\theta_{0,\ell}}(Z_i^{(\ell)})\}=0$, it is then natural to assume that $E\hat\theta_{j,\ell}-\theta_{j,\ell}=O_P(m^{-\gamma})$ for some positive constant $\gamma$ such that $1/2<\gamma\leq 1$
\end{remark}

\subsection{Proofs of Theorems}
\begin{proof}[Proof of Theorem~\ref{theo: asy npsvm}] From the proof of Theorem 6 in \cite{PSVM2011}, we know that
\begin{align}\nonumber
\hat\theta_{0,\ell}=&\theta_{0,\ell}+n^{-1} \sum_{i=1}^n S_{\theta_{0,\ell}}(Z_i^{(\ell)}) +o_P(n^{-1/2}),\\\label{eq:expan thetaj}
\hat\theta_{j,\ell}=&\theta_{j,\ell}+m^{-1} \sum_{i\in \calI_j} S_{\theta_{0,\ell}}(Z_i^{(\ell)}) +o_P(m^{-1/2}),
\end{align}
where $\ell=1,\ldots, R-1$ and $j=1,\ldots,k$.
Recall that $
\widehat M_j = \sum_{\ell=1}^{h-1}  \hat\psi_{j,\ell}\hat\psi^{\trans}_{j,\ell}$ and $E\{S_{\theta_{0,\ell}}(Z_i^{(\ell)})\}=0$, we conclude that the leading term of $E\widehat M_j-M_0$ is
$$\sum_{\ell=1}^{h-1}  \psi_{0,\ell}(E\hat\psi_{j,\ell}-\psi_{j,\ell})\trans+\sum_{\ell=1}^{h-1}  (E\hat\psi_{j,\ell}-\psi_{j,\ell})\psi^{\trans}_{0,\ell},$$
which is $O_P(m^{-\gamma})$ as $E\hat\theta_{j,\ell}-\theta_{j,\ell}=O_P(m^{-\gamma})$ by Assumption 6. We can further derive that $\cov\{\textup{vec}(\widehat M_j)\}=\Sigma_M/m+o(1/m)$ because
$m^{1/2}\textup{vec}(\widehat M_j -E \widehat M_j )\rightarrow N(0_{p^2},\Sigma_{M})$. Moreover, since $\{m^{1/2}\textup{vec}(\widehat M_1 -E \widehat M_1 ), \ldots, m^{1/2}\textup{vec}(\widehat M_k -E \widehat M_k )\}$ are i.i.d. $p^2\times 1$ vectors with mean zero and covariance matrix being $\Sigma_M+o(1)$, we apply the central limit theorem
and the {Slutsky’s} theorem to get
\begin{align*}
k^{1/2} \sum_{j=1}^k m^{1/2} \textup{vec} \{(\widehat M_j) - E (\widehat M_j)\}/k \rightarrow N(0_{p^2},\Sigma_M)
\end{align*}
in distribution as $k$ goes to infinity. On the other hand, we have
\begin{align*}
n^{1/2} \sum_{j=1}^k \textup{vec} \{E (\widehat M_j) - M_0\}/k = n^{1/2} \{E (\widehat M_j) - M_0\} =O_P(n^{1/2}m^{-\gamma})=o_P(1),
\end{align*}
under the condition that $n=o(m^{2\gamma})$. The  {asymptotic} distribution of $\widetilde M$ is then obtained by noting that
\begin{align*}
n^{1/2}\textup{vec}(\widetilde M-M_{0})
=k^{1/2} \sum_{j=1}^k m^{1/2} \textup{vec} \{(\widehat M_j) - E (\widehat M_j)\}/k + n^{1/2} \sum_{j=1}^k \textup{vec} \{E (\widehat M_j) - M_0\}/k.
\end{align*}
We then get the limiting distribution of $\hat V$ based on \cite{Bura2008} to complete the proof.
\end{proof}

\begin{proof}[Proof of Theorem~\ref{theo: asy rpsvm}]
  As the final target formulation {we want to obtain} is similar to its counterpart in Theorem \ref{theo: asy psvm}, we just need to demonstrate that the $\hat\theta_{(B),\ell}$  obtained by the refined estimation have the same asymptotic expansion as that of  the original estimation presented in Theorem 6 in \cite{PSVM2011}.
  We know the fact that
  \begin{equation*}
  \begin{gathered}
  {\theta _{0,\ell}} = {\left[ {\lambda {n^{ - 1}}\sum\limits_{i = 1}^n {{{\hat X}_i}\hat X_i^TH'\left\{ {g\left( {{{\hat X}_i},\tilde Y_i^{\left( \ell \right)},{{\hat \theta }_{\left( 0 \right),l}}} \right)/h} \right\}} /h + 2diag\left( {\hat \Sigma ,0} \right)} \right]^{ - 1}} \hfill \\
  \begin{array}{*{20}{c}}
  {}&{\left[ {\lambda {n^{ - 1}}\sum\limits_{i = 1}^n {\tilde Y_i^{\left( \ell \right)}{{\hat X}_i}\left( {\tilde Y_i^{\left( \ell \right)}{{\hat X}_i}{\theta _{0,l}}} \right)H'\left\{ {g\left( {{{\hat X}_i},\tilde Y_i^{\left( \ell \right)},{{\hat \theta }_{\left( 0 \right),l}}} \right)/h} \right\}} /h + 2\diag\left( {\hat \Sigma ,0} \right)} \right]}
  \end{array} {\theta _{0,\ell}} \hfill \\
  \end{gathered}
  \end{equation*}
Comparing the expression of $\theta_{0,\ell}$ with equation (\ref{eq: update_theta}), we obtain
  \begin{equation*}
  {{\hat \theta }_{\left( 1 \right),\ell}} - {\theta _{0,\ell}} = {H_{n,h,{\theta _{0,\ell}}}}^{ - 1}{D_{n,h,{\theta _{0,\ell}}}},
  \end{equation*}
where ${H_{n,h,{\theta _{0,\ell}}}}$ and ${D_{n,h,{\theta _{0,\ell}}}}$ are defined as
\begin{equation*}
\left\{ \begin{gathered}
{H_{n,h,{\theta _{0,\ell }}}} =  {n^{ - 1}}\sum\limits_{i = 1}^n {{{\hat X}_i}\hat X_i^TH'\left\{ {g\left( {{{\hat X}_i},\tilde Y_i^{\left( \ell  \right)},{{\hat \theta }_{\left( 0 \right),\ell }}} \right)/h} \right\}} /h + 2\diag\left( {\hat \Sigma ,0} \right)/{\lambda} \hfill \\
{D_{n,h,{\theta _{0,\ell}} }} =  {n^{ - 1}}\sum\limits_{i = 1}^n {\tilde Y_i^{\left( \ell  \right)}{{\hat X}_i}\left[ {H\left\{ {g\left( {{{\hat X}_i},\tilde Y_i^{\left( \ell  \right)},{{\hat \theta }_{\left( 0 \right),\ell }}} \right)/h} \right\} + \left\{ {g\left( {{{\hat X}_i},\tilde Y_i^{\left( \ell  \right)},{\theta _{0,\ell }}} \right)/h} \right\}} \right.}  \hfill \\
\begin{array}{*{20}{c}}
{}&{}&{}
\end{array}\left. {H'\left\{ {g\left( {{{\hat X}_i},\tilde Y_i^{\left( l \right)},{{\hat \theta }_{\left( 0 \right),l}}} \right)/h} \right\}} \right] - 2\diag\left( {\hat \Sigma ,0} \right){\theta _{0,\ell }}/{\lambda} \hfill \\
\end{gathered}  \right.
\end{equation*}

The following two propositions are necessary to complete the proof.

\begin{proposition}\label{prop: asp h} Assume the regularity conditions 1-5 and 7 in the Appendix hold true, we have
  \begin{equation*}
  {{H_{n,h,{\theta _{0,\ell}}}} - {H_{{\theta _{0,\ell}}}}} = {O_P}\left( {{{\left\{ {\log n/nh} \right\}}^{1/2}} + m^{ - 1/2} + h + {n^{ - 1/2}}/{\lambda}} \right)
  \end{equation*}
\end{proposition}

\begin{proposition}\label{prop: asp D} Assume the regularity conditions 1-5 and 7 in the Appendix hold true, we have
  \begin{equation*}
  {{D_{n,h,{\theta _{0,\ell }}}} + {D_{{\theta _{0,\ell }}}}( {{Z^{\left( \ell  \right)}}} )}  = {O_P}\left( {{{\left\{ {h\log n /n} \right\}}^{ - 1/2}} + {h^2} + m^{ - 1} + {n^{ - 1/2}}/{\lambda}} \right).
  \end{equation*}
\end{proposition}

Invoking the above two propositions, we can obtain
\begin{equation*}
{{\hat \theta }_{\left( 1 \right),l}} - {\theta _{0,\ell}} =  - H_{{\theta _0},\ell}^{ - 1}{D_{{\theta _0},\ell}}( {{Z^{\left( \ell \right)}}}) + {r_n},
\end{equation*}
where the order of the remainder $r_n$ can be derived as follows
\begin{equation*}
r_n = {O_P}\left( {{{\left\{ {{h}\log n/n} \right\}}^{1/2}}+{{\left\{ {\log n/hn^2} \right\}}^{1/2}}+ h^2+ m^{-1} + {n^{ - \tau }}+n^{-1/2}/{\lambda}} \right)
\end{equation*}
and $\tau > 1/2$ is specified in Assumption 6.
In general, for the $B$th {iteration} with  $h_B=\max \{ {{n^{ - 1/2}},{m^{ - {2^{B - 2}}}}} \}$, we have
\begin{equation*}
{{\hat \theta }_{\left( B \right),l}} - {\theta _{0,\ell}} =  - H_{{\theta _0},\ell}^{ - 1}{D_{{\theta _0},\ell}}( {{Z^{\left( \ell \right)}}} ) + {r_n},
\end{equation*}
where the remainder is
\begin{equation*}
r_n = {O_P}\left( {{{\left\{ {{h_B}\log n/n} \right\}}^{1/2}}+{{\left\{ {\log n/h_Bn^2} \right\}}^{1/2}}+ h_B^2 + {n^{ - \tau }}+n^{-1/2}/{\lambda}} \right)
\end{equation*}
With the assumption that $B\ge \lceil C_0\log_2(\log_m n) \rceil $, we see that $h_B^2=o_P(n^{-1/2})$. And ${{{\left\{ {{h_B}\log n/n} \right\}}^{1/2}}}=o_P(n^{-1/2})$ and ${{\left\{ {\log n/h_Bn^2} \right\}}^{1/2}}=o_P(n^{-1/2})$ under Assumption 7. Moreover, $n^{-1/2}/\lambda=o_P(n^{-1/2})$ as $\lambda\rightarrow\infty$. Then we  {conclude} that $r_n=o_P(n^{-1/2})$, which entails that
$${{\hat \theta }_{\left( B \right),l}} - {\theta _{0,\ell}} =  - H_{{\theta _0},\ell}^{ - 1}{D_{{\theta _0},\ell}}( {{Z^{\left( \ell \right)}}}) + o_P(n^{-1/2}). $$
We see that ${{\hat \theta }_{\left( 1 \right),\ell}}$ enjoys the same asymptotic expansion form as that of ${{\hat \theta }_{n,\ell}}$. The result is then {straightforward} following proof of Theorem 7 and Corollary 1 in \cite{PSVM2011}.
\end{proof}

\begin{proof}[Proof of Proposition~\ref{prop: asp h}]
Let $\delta(.)$ denote the Dirac delta function. Without loss of generality, assume $\mu=0$ is known, then $\hat X_i=\tilde X_i=(X\trans,-1)\trans$. By algebra calculations, we have
  \begin{eqnarray*}
    {{H_{n,h,{\theta _0},\ell }} - {H_{{\theta _0},\ell }}}  &=&  { {n^{ - 1}}\sum\limits_{i = 1}^n {{{\hat X}_i}\hat X_i^TH'\left\{ {g\left( {{{\tilde X}_i},\tilde Y_i^{\left( \ell  \right)},{{\hat \theta }_{\left( 0 \right),\ell }}} \right)/h} \right\}} /h - E\left[ {\delta \left( {1 - {{\tilde Y}^{\left( \ell  \right)}}{{\hat X}^T}{\theta _{0,\ell }}} \right)\tilde X{{\tilde X}\trans}} \right]} \\ \nonumber
    &&+ 2\diag\left( {\hat \Sigma ,0} \right)/{\lambda} - 2\diag\left( {\Sigma ,0} \right)/{\lambda} \\
    &=:& {{T_1}}  +  {{T_2}},
  \end{eqnarray*}
where $T_1$ and $T_2$ are defined as
  \begin{equation*}
  \left\{ \begin{gathered}
  {T_1} =  {n^{ - 1}}\sum\limits_{i = 1}^n {{{\hat X}_i}\hat X_i^TH'\left\{ {g\left( {{{\hat X}_i},\tilde Y_i^{\left( \ell  \right)},{{\hat \theta }_{\left( 0 \right),\ell }}} \right)/h} \right\}} /h - E\left[ {\delta \left( {1 - {{\tilde Y}^{\left( \ell  \right)}}{{\hat X}^T}{\theta _{0,\ell }}} \right)\hat X{{\hat X}^T}} \right] \hfill \\
  {T_2} = 2\diag\left( {\hat \Sigma ,0} \right)/{\lambda} - 2\diag\left( {\Sigma ,0} \right)/{\lambda} \hfill \\
  \end{gathered}  \right.
  \end{equation*}
  \par
Because $\hat\Sigma-\Sigma=O_P(n^{-1/2})$, then $T_2=O_P(n^{-1/2}/\lambda)$. In the next, we will deal with $T_1$. From the proof of lemma 3 in Cai et al. (2010), we have

\begin{equation*}
\left\| {{T_1}} \right\|  \leqslant 4\mathop {\sup }\limits_{1\le j \leq N_1} \left| {v_j^T{T_1}{v_j}} \right|.
\end{equation*}
where $v_j$'s are some non-random vectors with $\|v_j\|_2=1$ and $N_1$ is some positive constant. For $\alpha$ satisfying that $\alpha-\theta_{0,\ell}=O_P(m^{-1/2})$, we define
\begin{eqnarray*}
  {H_{n,h,j,{\theta _0,\ell} }}\left( \alpha  \right) &=& \frac{1}{{nh}}\sum\limits_{i = 1}^n {v_j^T{{\hat X}_i}\hat X_i^T{v_j}H'\left\{ {g\left( {{{\hat X}_i},\tilde Y_i^{\left( \ell  \right)},\alpha } \right)/h} \right\}}   \\
  &=& \frac{1}{{nh}}\sum\limits_{i = 1}^n {{{\left( {v_j^T{{\hat X}_i}} \right)}^2}H'\left\{ {g\left( {{{\hat X}_i},\tilde Y_i^{\left( \ell  \right)},\alpha } \right)/h} \right\}}.
\end{eqnarray*}
We then have
\begin{equation*}
\mathop {\sup }\limits_{1\le j \leq N_1} \left| {v_j^T{T_1}{v_j}} \right| \leqslant \mathop {\sup }\limits_{1\le j \leq  N_1 } \mathop {\sup }\limits_{ {\alpha  - {\theta _{0,\ell }}}  = O_P(m^{-1/2}) } \left| {{H_{n,h,j,{\theta _0},\ell }}\left( \alpha  \right) - v_j^TE\left[ {\delta \left( {1 - {{\tilde Y}^{\left( \ell  \right)}}{{\hat X}^T}{\theta _{0,l}}} \right)} \right]{v_j}} \right|.
\end{equation*}
Moreover, for any positive constant $C'>0$, we could form a sequence $\{\alpha_j,1 \leqslant k \leqslant  n^{C'}\}$ and further find an ${{\alpha _k}}$ in the sequence  {satisfying} that
\begin{equation*}
{\alpha  - {\alpha _k}}  = {O_P}\left( {{m^{ - 1/2}}/{n^{C'}}} \right).
\end{equation*}

Then, with the asymptotic property of $\alpha$, we can carry out the following three lemmas.

\begin{lemma}\label{lemma: heh-discrete}
Under the assumptions in theorem~\ref{theo: asy rpsvm} and with the asymptotic property of  $\alpha$, we have
\begin{equation}
\begin{gathered}
  \mathop {\sup }\limits_j \mathop {\sup }\limits_{\left\| {\alpha  - {\theta _{0,\ell }}} \right\| \leqslant {m^{ - 1/2}}} \left| {{H_{n,h,j,{\theta _{0,\ell }}}}\left( \alpha  \right) - v_j^TE\left[ {\delta \left( {1 - {{\tilde Y}^{\left( \ell  \right)}}{{\hat X}^T}{\theta _{0,\ell }}} \right)} \right]{v_j}} \right| \hfill \\
  \begin{array}{*{20}{c}}
  {}&{}&{}
\end{array} - \mathop {\sup }\limits_j \mathop {\sup }\limits_{k \leqslant {n^{C'}}} \left| {{H_{n,h,j,{\theta _{0,\ell }}}}\left( {{\alpha _k}} \right) - v_j^TE\left[ {\delta \left( {1 - {{\tilde Y}^{\left( \ell  \right)}}{{\hat X}^T}{\theta _{0,\ell }}} \right)} \right]{v_j}} \right| =  {o_P}({n^{ - 1/2}}) \hfill \\
\end{gathered}
\end{equation}
\end{lemma}

\begin{lemma}\label{lemma: h-eh}
Under the assumptions in theorem~\ref{theo: asy rpsvm} and with the asymptotic property of  $\alpha$, we have
\begin{equation}
 \mathop {\sup }\limits_j \mathop {\sup }\limits_{k \leqslant {n^{C'}}} \left| {{H_{n,h,j,{\theta _{0,\ell }}}}\left( {{\alpha _k}} \right) - E\left[ {{H_{n,h,j,{\theta _{0,\ell }}}}\left( {{\alpha _k}} \right)} \right]} \right| = {O_P}\left( {{{\left( {\log n/nh} \right)}^{1/2}}} \right)
\end{equation}
\end{lemma}

\begin{lemma}\label{lemma: eh-vev}
Under the assumptions in theorem~\ref{theo: asy rpsvm} and with the asymptotic property of  $\alpha$, we have
\begin{equation}
E\left[ {{H_{n,h,j,{\theta _{0,\ell }}}}\left( {{\alpha _k}} \right)} \right] - v_j^TE\left[ {\delta \left( {1 - {{\tilde Y}^{\left( \ell  \right)}}{{\hat X}^T}{\theta _{0,\ell }}} \right)} \right]{v_j} = O\left( {h + {m^{ - 1/2}}} \right)
\end{equation}
\end{lemma}


Combine these lemmas, we get
\begin{equation*}
{{T_1}}  = {O_P}\left( {{{\left\{ {\log n/nh} \right\}}^{1/2}} + h +  m^{-1/2} } \right).
\end{equation*}
The proof is completed.
\end{proof}

\begin{proof}[Proof of Proposition~\ref{prop: asp D}]
  By some algebra calculations, we have
  \begin{equation*}
  {{D_{n,h,{\theta _{0,\ell }}}} + {D_{{\theta _{0,\ell }}}}( {{Z^{\left( \ell  \right)}}} )}  =  {{T_3} + {T_4}} ,
  \end{equation*}
  where $T_3$ and $T_4$ are defined as
  \begin{equation*}
  \left\{ \begin{gathered}
  {T_3} =  {n^{ - 1}}\sum\limits_{i = 1}^n {\tilde Y_i^{\left( \ell  \right)}{{\hat X}_i}\left[ {H\left\{ {g\left( {{{\hat X}_i},\tilde Y_i^{\left( \ell  \right)},{{\hat \theta }_{\left( 0 \right),\ell }}} \right)/h} \right\} + \left\{ {g\left( {{{\hat X}_i},\tilde Y_i^{\left( \ell  \right)},{\theta _{0,\ell }}} \right)/h} \right\} \cdot } \right.}  \hfill \\
  \begin{array}{*{20}{c}}
  {}&{}&{\left. {H'\left\{ {g\left( {{{\hat X}_i},\tilde Y_i^{\left( \ell  \right)},{{\hat \theta }_{\left( 0 \right),\ell }}} \right)/h} \right\}} \right] -  {E_n}\left\{ {\hat X{{\tilde Y}^{\left( \ell  \right)}}I\left( {1 - \theta _{0,\ell }^T\hat X{{\tilde Y}^{\left( \ell  \right)}} > 0} \right)} \right\}}
  \end{array} \hfill \\
  {T_4} = 2\diag\left( {\Sigma  - \hat \Sigma ,0} \right){\theta _{0,\ell }}/{\lambda} \hfill \\
  \end{gathered}  \right.
  \end{equation*}
Again $T_4=O_P(n^{-1/2}/\lambda)$. We will calculate the order of $T_3$ in the following. We define
\begin{equation*}
\begin{gathered}
{T_3}\left( \alpha  \right) =  {n^{ - 1}}\sum\limits_{i = 1}^n {\tilde Y_i^{\left( \ell  \right)}{{\hat X}_i}\left[ {H\left\{ {g\left( {{{\hat X}_i},\tilde Y_i^{\left( \ell  \right)},\alpha } \right)/h} \right\} + \left\{ {g\left( {{{\hat X}_i},\tilde Y_i^{\left( \ell  \right)},{\theta _{0,\ell }}} \right)/h} \right\} \cdot } \right.}  \hfill \\
\begin{array}{*{20}{c}}
{}&{}&{\left. {H'\left\{ {g\left( {{{\hat X}_i},\tilde Y_i^{\left( \ell  \right)},\alpha } \right)/h} \right\} - I\left( {1 - \theta _{0,\ell }^T\hat X{{\tilde Y}^{\left( \ell  \right)}} > 0} \right)} \right]} ,
\end{array} \hfill \\
\end{gathered}
\end{equation*} 
Based on the above  {definition}, we can get that
\begin{equation*}
\left\| {T_3}\right\| = \left\| {{T_3}( {{{\hat \theta }_{\left( 0 \right),\ell }}} )} \right\|: = \mathop {\sup }\limits_{v \in {R^{p + 1}},{{\left\| v \right\|}_2} = 1} \left\| {{T_3}( {{{\hat \theta }_{\left( 0 \right),\ell }}} )v} \right\|.
\end{equation*}
As we don't know the optimal $v$, we make a $1/2$-net of the unit sphere $S^{p+1}$ in the Euclidean distance in $R^{p+1}$ and denote it by $S_{1/2}^{p+1}$. From Roman Vershynin (2011), we have $K_0=:\text{Card}( {S_{1/2}^{p+1}})\le N_1$.
Let ${v_1}, \ldots ,{v_{K_0} }$ be the centers of these $K_0$ elements in the net. Then for any $v \in {R^{p + 1}}$, there exists a ${v_j} \in \left\{ {{v_1}, \ldots ,{v_{K_0} }} \right\}$ such that ${\left\| {v - {v_j}} \right\|_2} \leqslant 1/2$. Then
\begin{equation*}
\left\| {{T_3}( {{{\hat \theta }_{\left( 0 \right),\ell }}} )} \right\| \leqslant 2\mathop {\sup }\limits_{1\le j \le  K_0} \left| {{T_3}( {{{\hat \theta }_{\left( 0 \right),\ell }}} ){v_j}} \right|.
\end{equation*}
We define
\begin{equation*}
\begin{gathered}
{T_{3,j}}\left( \alpha  \right) =  {n^{ - 1}}\sum\limits_{i = 1}^n {v_j^T\tilde Y_i^{\left( \ell  \right)}{{\hat X}_i}\left[ {H\left\{ {g\left( {{{\hat X}_i},\tilde Y_i^{\left( \ell  \right)},\alpha } \right)/h} \right\} + \left\{ {g\left( {{{\hat X}_i},\tilde Y_i^{\left( \ell  \right)},{\theta _{0,\ell }}} \right)/h} \right\}  } \right.}  \hfill \\
\begin{array}{*{20}{c}}
{}&{}&{\left. {H'\left\{ {g\left( {{{\hat X}_i},\tilde Y_i^{\left( \ell  \right)},\alpha } \right)/h} \right\} - I\left( {1 - \theta _{0,\ell }^T\hat X{{\tilde Y}^{\left( \ell  \right)}} > 0} \right)} \right]}.
\end{array} \hfill \\
\end{gathered}
\end{equation*}
Then we can bound $T_3$ as
\begin{equation*}
\left\| {{T_3}} \right\| \le \mathop {\sup }\limits_{ {\alpha  - {\theta _{0,\ell }}} = O_P(m^{-1/2})} \left\| {{T_3}\left( \alpha  \right)} \right\| \leqslant 2\mathop {\sup }\limits_{1\le j\le K_0} \mathop {\sup }\limits_{{\alpha  - {\theta _{0,\ell }}} = O_P(m^{-1/2})} \left| {{T_{3,j}}\left( \alpha  \right)} \right|.
\end{equation*}
Similar to the proof of Proposition 2, we can also obtain the following three lemmas.

\begin{lemma}\label{lemma: tet-discrete}
Under the assumptions in theorem~\ref{theo: asy rpsvm} and with the asymptotic property of  $\alpha$, we have
\begin{equation}
\mathop {\sup }\limits_j \mathop {\sup }\limits_{\alpha  - {\theta _{0,\ell }} = {O_P}({m^{ - 1/2}})} \left| {{T_{3,j}}\left( \alpha  \right)} \right| - \mathop {\sup }\limits_j \mathop {\sup }\limits_{1 \leqslant k \leqslant {n^{C'}}} \left| {{T_{3,j}}\left( {{\alpha _k}} \right)} \right| = {o_P}({n^{ - 1/2}}).
\end{equation}
\end{lemma}

\begin{lemma}\label{lemma: t-et}
Under the assumptions in theorem~\ref{theo: asy rpsvm} and with the asymptotic property of  $\alpha$, we have
\begin{equation}
\mathop {\sup }\limits_j \mathop {\sup }\limits_{1\le k \le n^{C'}} {\left| {{T_{3,j}}\left( {{\alpha _k}} \right) - E{T_{3,j}}\left( {{\alpha _k}} \right)} \right|}  = {O_P}\left( {{\left\{ {h\log n/n} \right\}}^{1/2}} \right).
\end{equation}
\end{lemma}

\begin{lemma}\label{lemma: et}
Under the assumptions in theorem~\ref{theo: asy rpsvm} and with the asymptotic property of  $\alpha$, we have
\begin{equation}
\mathop {\sup }\limits_j \mathop {\sup }\limits_{1\le k \le n^{C'}}E{T_{3,j}}\left( {{\alpha _k}} \right) = {O}\left( {{h^2} + {{\left\| {{\alpha _k} - {\theta _{0,\ell }}} \right\|}^2}} \right) = {O}\left( {{h^2} + m^{-1}} \right).
\end{equation}
\end{lemma}


Therefore, combining three lemmas above, we can obtain the conclusion that ${{T_3}}  = {O_P}\left( {{{\left\{ {h\log n/n} \right\}}^{1/2}} + {h^2} + m^{-1} } \right).$ The proof is completed by combining the results of $T_3$ and $T_4$.
\end{proof}

\begin{proof}[Proof of Theorem~\ref{theo: asy nwpsvm}]
From the theorem 3 in \cite{WPSVM2016} and similar with the proof of theorem 2, we can also draw that
\begin{equation}
{\hat \theta _{j,\ell }} = {\theta _{j,\ell }} + {m^{ - 1}}\sum\limits_{i \in {\calI_j}} {{\tilde S_{{\theta _{0,\ell }}}}} (Z_i^{(\ell )}) + {o_P}({m^{ - 1/2}})
\end{equation}
despite of a difference on ${S_{{\theta _{0,\ell }}}}\left( {Z_i^{(\ell )}} \right)$ in the theorem~\ref{theo: asy npsvm}. Hence, take advantage of the proof of theorem~\ref{theo: asy npsvm}, we only need to show $E\left\{ {{{\tilde S}_{{\theta _{0,\ell }}}}\left( {Z_i^{(\ell )}} \right)} \right\} = 0$ also holds true under the weighted PSVM scenario. This is obvious, as $\left| {{w_\pi }\left( y \right)} \right| \leqslant 1$.
\end{proof}

\begin{proof}[Proof of Theorem~\ref{theo: asy rwpsvm}]
Similar with the proof of theorem~\ref{theo: asy rpsvm}, we take derivative of equation \eqref{smooth:wpsvm} and input a good initial ${{\tilde \theta }_{\left( 0 \right),\ell }}$ which is the solution of WPSVM on any single machine, then
\begin{equation}
{{\tilde \theta }_{\left( 1 \right),\ell }} - {\theta _{0,\ell }} = \tilde H_{n,h,{\theta _{0,\ell }}}^{ - 1}{{\tilde D}_{n,h,{\theta _{0,\ell }}}},
\end{equation}
where ${{\tilde H}_{n,h,{\theta _{0,\ell }}}}$ and ${{\tilde D}_{n,h,{\theta _{0,\ell }}}}$ are defined as
\begin{equation}
\left\{ \begin{gathered}
  {{\tilde H}_{n,h,{\theta _{0,\ell }}}} = \lambda {n^{ - 1}}\sum\limits_{i = 1}^n {{w_\pi }\left( {{Y_i}} \right)} {{\hat X}_i}\hat X_i^TH'\left\{ {g\left( {{\hat X_i},{Y_i},{{\tilde \theta }_{\left( 0 \right),\ell }}} \right)/h} \right\}/h + 2diag\left( {\hat \Sigma ,0} \right) \hfill \\
  {{\tilde D}_{n,h,{\theta _{0,\ell }}}} = \lambda {n^{ - 1}}\sum\limits_{i = 1}^n {{w_\pi }\left( {{Y_i}} \right)} {{\hat X}_i}{Y_i}\left[ {H\left\{ {g\left( {{\hat  X_i},{Y_i},{{\tilde \theta }_{\left( 0 \right),\ell }}} \right)/h} \right\} + \left\{ {g\left( {{\hat X_i},{Y_i},{\theta _{0,\ell }}} \right)/h} \right\} \cdot } \right. \hfill \\
  \begin{array}{*{20}{c}}
  {}&{}&{}
\end{array}H'\left\{ {g\left( {{\hat X_i},{Y_i},{{\tilde \theta }_{\left( 0 \right),\ell }}} \right)/h} \right\} - 2diag\left( {\hat \Sigma ,0} \right){\theta _{0,\ell }} \hfill \\
\end{gathered}  \right.
\end{equation}
And naturally, we want to show the following two propositions hold true, which can directly lead to the final conclusion.

\begin{proposition}\label{prop: asp wh} Assume the regularity conditions 1-5 and 7 in the Appendix hold true, we have
  \begin{equation*}
  {{\tilde H}_{n,h,{\theta _{0,\ell }}}} - {{\tilde H}_{{\theta _{0,\ell }}}} = {O_P}\left( {{{\left\{ {\log n/nh} \right\}}^{1/2}} + {m^{ - 1/2}} + h + {n^{ - 1/2}}/\lambda } \right)
  \end{equation*}
\end{proposition}

\begin{proposition}\label{prop: asp wD} Assume the regularity conditions 1-5 and 7 in the Appendix hold true, we have
  \begin{equation*}
  {{\tilde D_{n,h,{\theta _{0,\ell }}}} + {\tilde D_{{\theta _{0,\ell }}}}( {{Z^{\left( \ell  \right)}}} )}  = {O_P}\left( {{{\left\{ {h\log n /n} \right\}}^{ - 1/2}} + {h^2} + m^{ - 1} + {n^{ - 1/2}}/{\lambda}} \right).
  \end{equation*}
\end{proposition}

Then, the rest of the proof can be founds in the proof of theorem~\ref{theo: asy rpsvm}.

\end{proof}

\subsection{Proofs of Lemmas}
\begin{proof}[Proof of Lemma\ref{lemma: heh-discrete}]
Noticing that ${\left\| {\alpha  - {\theta _{0,\ell }}} \right\|_2} \leqslant {O_P}\left( {{m^{ - 1/2}}} \right)$, we construct a set of vectors $\left\{ {{\alpha _k},1 \leqslant k \leqslant {n^{M\left( {p + 1} \right)}}} \right\}$ in ${R^{p + 1}}$ by dividing each $\left[ {{\theta _{0,{\ell _i}}} - {m^{ - 1/2}},{\theta _{0,{\ell _i}}} + {m^{ - 1/2}}} \right]$ into ${n^M}$ small equal pieces. Hence, for any possible vector $\alpha$ in the ball ${\left\| {\alpha  - {\theta _{0,\ell }}} \right\|_2} \leqslant {O_P}\left( {{m^{ - 1/2}}} \right)$, there exist $\Lambda  \subset \left[ {{n^{C'}}} \right]$ where $C' = M\left( {p + 1} \right)$ is a constant such that for any $k \in \Lambda $, we have ${\left\| {\alpha  - {\alpha _k}} \right\|_2} \leqslant 2{\left( {p + 1} \right)^{1/2}}{m^{ - 1/2}}/{n^M}$ . Then, combining the Lipschitzness of $H'\left( x \right)$, we can obtain
\begin{equation}
\begin{gathered}
  \left| {{{\left( {v_j^T{{\hat X}_i}} \right)}^2}H'\left\{ {g\left( {{{\hat X}_i},\tilde Y_i^{\left( \ell  \right)},\alpha } \right)/h} \right\} - {{\left( {v_j^T{{\hat X}_i}} \right)}^2}H'\left\{ {g\left( {{{\hat X}_i},\tilde Y_i^{\left( \ell  \right)},{\alpha _k}} \right)/h} \right\}} \right| \hfill \\
   \leqslant C{\left( {p + 1} \right)^{1/2}}{m^{ - 1/2}}\left\| {{{\hat X}_i}} \right\|_2^3/h{n^M}, \hfill \\
\end{gathered}
\end{equation}
which is due to ${\left\| {\tilde Y_i^{\left( \ell  \right)}} \right\|_2} = 1 = {\left\| {{v_j}} \right\|_2}$. Therefore, combined with the definition of ${H_{n,h,j,{\theta _{0,\ell }}}}$, we have
\begin{equation}
\begin{gathered}
  \mathop {\sup }\limits_j \mathop {\sup }\limits_{\left\| {\alpha  - {\theta _{0,\ell }}} \right\| \leqslant {m^{ - 1/2}}} \left| {{H_{n,h,j,{\theta _{0,\ell }}}}\left( \alpha  \right) - v_j^TE\left[ {\delta \left( {1 - {{\tilde Y}^{\left( \ell  \right)}}{{\hat X}^T}{\theta _{0,\ell }}} \right)} \right]{v_j}} \right| -  \hfill \\
  \begin{array}{*{20}{c}}
  {}&{}&{}
\end{array}\mathop {\sup }\limits_j \mathop {\sup }\limits_{k \leqslant {n^{C'}}} \left| {{H_{n,h,j,{\theta _{0,\ell }}}}\left( {{\alpha _k}} \right) - v_j^TE\left[ {\delta \left( {1 - {{\tilde Y}^{\left( \ell  \right)}}{{\hat X}^T}{\theta _{0,\ell }}} \right)} \right]{v_j}} \right| \hfill \\
   \leqslant \sum\limits_{i = 1}^n {C\lambda {{\left( {p + 1} \right)}^{1/2}}{m^{ - 1/2}}\left\| {{{\hat X}_i}} \right\|_2^3/{n^{\left( {M + 1} \right)}}{h^2}}.  \hfill \\
\end{gathered}
\end{equation}
We can easily get the conclusion of lemma 1 when considering $n \to \infty $ with the Assumption 1 about the boundary of the norm of $X$.
\end{proof}

\begin{proof}[Proof of Lemma\ref{lemma: h-eh}]
Take
\[\begin{gathered}
  {\xi _{ij}} = {\left( {v_j^T{{\hat X}_i}} \right)^2}H'\left( {g\left\{ {{{\hat X}_i},\tilde Y_i^{\left( \ell  \right)},{\theta _{0,\ell }}} \right\}/h - \tilde Y_i^{\left( \ell  \right)}\hat X_i^T\omega /h} \right) \hfill \\
   = {\left( {v_j^T{{\hat X}_i}} \right)^2}H'\left( {({\varepsilon _i} - \tilde Y_i^{\left( \ell  \right)}\hat X_i^T\omega) /h} \right) \hfill \\
\end{gathered} \]
where
\[\left\{ \begin{gathered}
  {\varepsilon _i} = g\left\{ {{{\hat X}_i},\tilde Y_i^{\left( \ell  \right)},{\theta _{0,\ell }}} \right\} \hfill \\
  \omega  = \alpha  - {\theta _{0,\ell }} \hfill \\
\end{gathered}.  \right.\]
Then, the following inequalities stand
\begin{equation}
\begin{gathered}
  E\left[ {\xi _{ij}^2\exp \left( {t\left| {{\xi _{ij}}} \right|} \right)} \right] \leqslant E\left[ {\xi _{ij}^2\exp \left( {Ct{{\left( {v_j^T{{\hat X}_i}} \right)}^2}} \right)} \right] \hfill \\
   \leqslant CE\left[ {{{\left( {v_j^T{{\hat X}_i}} \right)}^2}H'{{\left( {{\varepsilon _i} - \tilde Y_i^{\left( \ell  \right)}{{\hat X}_i}\omega } \right)}^2}} \right]. \hfill \\
\end{gathered}
\end{equation}
Next, very technically, we decompose the expectation, the calculus over $x \in {R^p}$, into a double integration with $z$ and ${x_{ - 1}}$ separately, where $x = {\left( {z,x_{ - 1}^T} \right)^T}$. Specifically,
\begin{equation}
\begin{gathered}
  E\left[ {{{\left( {v_j^T{{\hat X}_i}} \right)}^2}H'{{\left( ({{\varepsilon _i} - \tilde Y_i^{\left( \ell  \right)}{{\hat X}_i}\omega })/h \right)}^2}} \right] \hfill \\
   =  - h/{\omega _1}\int_{{R^{p - 1}}} {{f_{ - 1}}\left( {{x_{ - 1}}} \right)\int_R {\left( {{v^T}\hat x} \right)^2} } H'{\left( z \right)^2}f\left( {\left( {1 - {\omega _0} - x_{ - 1}^T{\omega _{ - 1}} - hz} \right)/{\omega _1}\left| {{x_{ - 1}}} \right.} \right)dzd{x_{ - 1}} \hfill \\
   = O\left( h \right), \hfill \\
\end{gathered}
\end{equation}
where the Lipschitzness of $H'$ and $f$ is used in the last inequality. Hence,
\begin{equation}
E\left[ {\xi _{ij}^2\exp \left( {t\left| {{\xi _{ij}}} \right|} \right)} \right] \leqslant CE\left[ {{{\left( {v_j^T{{\hat X}_i}} \right)}^2}H'{{\left( {{\varepsilon _i} - \tilde Y_i^{\left( \ell  \right)}{{\hat X}_i}\omega } \right)}^2}} \right] = O\left( h \right).
\end{equation}
And finally, adopting the Lemma 1 in \citep{Cai2011Adaptive}, we have for any $\gamma  > 0$,

\begin{equation}
 \mathop {\sup }\limits_j \mathop {\sup }\limits_{k \leqslant {n^{C'}}} P\left(\left| {{H_{n,h,j,{\theta _{0,\ell }}}}\left( {{\alpha _k}} \right) - E\left[ {{H_{n,h,j,{\theta _{0,\ell }}}}\left( {{\alpha _k}} \right)} \right]} \right|\geq C{\left( {\log n/nh} \right)}^{1/2}\right) = {O}\left( n^{-\gamma}\right)
\end{equation}
\end{proof}

\begin{proof}[Proof of Lemma\ref{lemma: eh-vev}]
Notice that
\begin{equation}
\label{eh-vev}
\begin{gathered}
  E\left[ {{H_{n,h,j,{\theta _{0,\ell }}}}\left( {{\alpha _k}} \right)} \right] = E\left\{ {{{\left( {{v^T}\hat X} \right)}^2}H'\left\{ {g\left( {\hat X,{{\tilde Y}^{\left( \ell  \right)}},{\alpha _k}} \right)/h} \right\}/h} \right\} \hfill \\
  \mathop  = \limits^{\left( a \right)} P\left[ {{{\tilde Y}^{\left( \ell  \right)}} = 1} \right]\int_{{R^p}} {{{\left( {{v^T}\hat x} \right)}^2}H'\left\{ {\left( {1 - {{\hat x}^T}{\alpha _k}} \right)/h} \right\}f\left( x \right)/hdx}  +  \hfill \\
  \begin{array}{*{20}{c}}
  {}&{}
\end{array}P\left[ {{{\tilde Y}^{\left( \ell  \right)}} =  - 1} \right]\int_{{R^p}} {{{\left( {{v^T}\hat x} \right)}^2}H'\left\{ {\left( {1 + {{\hat x}^T}{\alpha _k}} \right)/h} \right\}g\left( x \right)/hdx}  \hfill \\
   = P\left[ {{{\tilde Y}^{\left( \ell  \right)}} = 1} \right]{T^{\left(  +  \right)}} + P\left[ {{{\tilde Y}^{\left( \ell  \right)}} =  - 1} \right]{T^{\left(  -  \right)}}, \hfill \\
\end{gathered}
\end{equation}
where
\[\left\{ \begin{gathered}
  {T^{\left(  +  \right)}} = \int_{{R^p}} {{{\left( {{v^T}\hat x} \right)}^2}H'\left\{ {\left( {1 - {{\hat x}^T}{\alpha _k}} \right)/h} \right\}f\left( x \right)/hdx}  \hfill \\
  {T^{\left(  -  \right)}} = \int_{{R^p}} {{{\left( {{v^T}\hat x} \right)}^2}H'\left\{ {\left( {1 + {{\hat x}^T}{\alpha _k}} \right)/h} \right\}g\left( x \right)/hdx}  \hfill \\
\end{gathered}.  \right.\]
Plus, in the equation (a) above, ${f\left( x \right)}$ is the sample distribution when its corresponding $Y =1$ and ${g\left( x \right)}$ is the sample distribution when its corresponding $Y = -1$.
\par
Then firstly expanding these distributions to joint distributions of its first element and the rest, take ${T^{\left(  +  \right)}}$ as an example,
\begin{equation}
\begin{gathered}
  {T^{\left(  +  \right)}} = \int_{{R^p}} {{{\left( {{v^T}\hat x} \right)}^2}H'\left\{ {\left( {1 - {{\hat x}^T}{\alpha _k}} \right)/h} \right\}f\left( x \right)/hdx}  \hfill \\
   = \left( {1/h} \right)\int_{{R^{p - 1}}} {\int_R {\left( {{v_0} + {v_1}{x_1} + v_{ - 1}^T{x_{ - 1}}} \right)^2H'\left\{ {\left( {1 - {\alpha _{k,0}} - {x_1}{\alpha _{k,1}} -  x_{ - 1}^T{\alpha _{k, - 1}}} \right)/h} \right\}f\left( {{x_1},{x_{ - 1}}} \right)} } d{x_1}d{x_{ - 1}} \hfill \\
   = \left(- 1/{\alpha _{k,1}}\right)\int_{{R^{p - 1}}} {{f_{ - 1}}\left( {{x_{ - 1}}} \right)} {\int_{ - 1}^1 {\left( {{v_0} + {v_1}\left( {1 - {\alpha _{k,0}} - x_{ - 1}^T{\alpha _{k, - 1}} - hy} \right)/{\alpha _{k,1}} + v_{ - 1}^T{x_{ - 1}}} \right)} ^2} \times  \hfill \\
  \begin{array}{*{20}{c}}
  {}&{}&{}
\end{array}f\left( {\left\{ {1 - {\alpha _{k,0}} -  x_{ - 1}^T{\alpha _{k, - 1}} - hy} \right\}/{\alpha _{k,1}}\left| {{x_{ - 1}}} \right.} \right)H'\left( y \right)dyd{x_{ - 1}}, \hfill \\
\end{gathered}
\end{equation}
where
\[y = \left( {1 - {\alpha _{k,0}} - {x_1}{\alpha _{k,1}} - x_{ - 1}^T{\alpha _{k, - 1}}} \right)/h\].
Then, with the fact that
\begin{equation}
\begin{gathered}
  {\left\{ {{v_0} + {v_1}\left( {1 - {\alpha _{k,0}} -  x_{ - 1}^T{\alpha _{k, - 1}} - hy} \right)/{\alpha _{k,1}} + v_{ - 1}^T{x_{ - 1}}} \right\}^2} \hfill \\
   = {\left( {{v_0} + v_{ - 1}^T{x_{ - 1}}} \right)^2} + 2\left( {{v_0} + v_{ - 1}^T{x_{ - 1}}} \right){v_1}\left( {1 - {\alpha _{k,0}} -  x_{ - 1}^T{\alpha _{k, - 1}} - hy} \right)/{\alpha _{k,1}} +  \hfill \\
  \begin{array}{*{20}{c}}
  {}&{}&{}
\end{array}v_1^2{\left( {1 - {\alpha _{k,0}} -  x_{ - 1}^T{\alpha _{k, - 1}} - hy} \right)^2}/\alpha _{k,1}^2 \hfill \\
   = {A_1} + {A_2} + {A_3}, \hfill \\
\end{gathered}
\end{equation}
we can expand
\begin{equation}
\begin{gathered}
  {T^{\left(  +  \right)}} =  - 1/{\alpha _{k,1}}\int_{{R^{p - 1}}} {{f_{ - 1}}\left( {{x_{ - 1}}} \right)} \left\{ {T_1^{\left(  +  \right)} + T_2^{\left(  +  \right)} + T_3^{\left(  +  \right)}} \right\}d{x_{ - 1}} \hfill \\
  T_i^{\left(  +  \right)} = \int_{ - 1}^1 {{A_i}f\left( {\left\{ {1 - {\alpha _{k,0}} - x_{ - 1}^T{\alpha _{k, - 1}} - hy} \right\}/{\alpha _{k,1}}\left| {{x_{ - 1}}} \right.} \right)H'\left( y \right)dy}.  \hfill \\
\end{gathered}
\end{equation}
And these $T_i^{\left(  +  \right)}$ can be solved separately.
\begin{equation*}
\begin{array}{l}
T_1^{\left(  +  \right)} = \int_{ - 1}^1 {{A_1}f\left( {\left\{ {1 - {\alpha _{k,0}} - x_{ - 1}^T{\alpha _{k, - 1}} - hy} \right\}/{\alpha _{k,1}}\left| {{x_{ - 1}}} \right.} \right)H'\left( y \right)dy} \\
 = \int_{ - 1}^1 {{{\left( {{v_0} + v_{ - 1}^T{x_{ - 1}}} \right)}^2}f\left( {\left\{ {\left( {1 - {\theta _{0,\ell ,0}} - x_{ - 1}^T{\theta _{0,\ell , - 1}}} \right)/{\theta _{0,\ell ,1}}} \right\}\left| {{x_{ - 1}}} \right.} \right)H'\left( y \right)dy} \\
\begin{array}{*{20}{c}}
{}&{ + O\left( 1 \right)\int_{ - 1}^1 {{{\left( {{v_0} + v_{ - 1}^T{x_{ - 1}}} \right)}^2}} \left\{ {\left( {1 - {\alpha _{k,0}} - x_{ - 1}^T{\alpha _{k, - 1}} - hy} \right)/{\alpha _{k,1}}} \right\}}
\end{array}\\
\begin{array}{*{20}{c}}
{}&{ - \left\{ {\left( {1 - {\theta _{0,\ell ,0}} - x_{ - 1}^T{\theta _{0,\ell , - 1}}} \right)/{\theta _{0,\ell ,1}}} \right\} \times \left| {H'\left( y \right)} \right|dy}
\end{array}
\end{array}
\end{equation*}
\[\begin{array}{l}
T_2^{\left(  +  \right)} = \int_{ - 1}^1 {{A_2}f\left( {\left\{ {1 - {\alpha _{k,0}} - x_{ - 1}^T{\alpha _{k, - 1}} - hy} \right\}/{\alpha _{k,1}}\left| {{x_{ - 1}}} \right.} \right)H'\left( y \right)dy} \\
 = \int_{ - 1}^1 {2\left( {{v_0} + v_{ - 1}^T{x_{ - 1}}} \right){v_1}\left\{ {\left( {1 - {\theta _{0,\ell ,0}} - x_{ - 1}^T{\theta _{0,\ell , - 1}}} \right)/{\theta _{0,\ell ,1}}} \right\} \times } \\
\begin{array}{*{20}{c}}
{}&{f\left( {\left\{ {\left( {1 - {\theta _{0,\ell ,0}} - x_{ - 1}^T{\theta _{0,\ell , - 1}}} \right)/{\theta _{0,\ell ,1}}} \right\}\left| {{x_{ - 1}}} \right.} \right)H'\left( y \right)dy}
\end{array}\\
\begin{array}{*{20}{c}}
{}&{ + O\left( 1 \right)\int_{ - 1}^1 {2\left( {{v_0} + v_{ - 1}^T{x_{ - 1}}} \right){v_1}} \left( {\left\{ {\left( {1 - {\alpha _{k,0}} - x_{ - 1}^T{\alpha _{k, - 1}} - hy} \right)/{\alpha _{k,1}}} \right\}} \right.}
\end{array}\\
\begin{array}{*{20}{c}}
{}&{\left. { - \left\{ {\left( {1 - {\theta _{0,\ell ,0}} - x_{ - 1}^T{\theta _{0,\ell , - 1}}} \right)/{\theta _{0,\ell ,1}}} \right\}} \right) \times \left| {H'\left( y \right)} \right|dy}
\end{array}
\end{array}\]
\[\begin{gathered}
  T_3^{\left(  +  \right)} = \int_{ - 1}^1 {{A_3}f\left( {\left\{ {1 - {\alpha _{k,0}} - x_{ - 1}^T{\alpha _{k, - 1}} - hy} \right\}/{\alpha _{k,1}}\left| {{x_{ - 1}}} \right.} \right)H'\left( y \right)dy}  \hfill \\
   = \int_{ - 1}^1 {v_1^2{{\left\{ {\left( {1 - {\theta _{0,\ell,0}} - x_{ - 1}^T{\theta _{0,\ell , - 1}}} \right)/{\theta _{0,\ell ,1}}} \right\}}^2}f\left( {\left\{ {\left( {1 - {\theta _{0,\ell,0 }} - x_{ - 1}^T{\theta _{0,\ell , - 1}}} \right)/{\theta _{0,\ell ,1}}} \right\}\left| {{x_{ - 1}}} \right.} \right)H'\left( y \right)dy}  \hfill \\
  \begin{array}{*{20}{c}}
  {}&\begin{gathered}
   + O\left( 1 \right)\int_{ - 1}^1 {v_1^2} \left( {\left\{ {\left( {1 - {\alpha _{k,0}} -  x_{ - 1}^T{\alpha _{k, - 1}} - hy} \right)/{\alpha _{k,1}}} \right\} - \left\{ {\left( {1 - {\theta _{0,\ell,0 }} - x_{ - 1}^T{\theta _{0,\ell , - 1}}} \right)/{\theta _{0,\ell ,1}}} \right\}} \right) \hfill \\
   \times \left| {H'\left( y \right)} \right|dy \hfill \\
\end{gathered}
\end{array} \hfill \\
\end{gathered} \]
\par
And combining the three integrals above, we have
\begin{equation}
\begin{gathered}
  {T^{\left(  +  \right)}} = \left( { - 1/{\alpha _{k,1}}} \right)\int_{{R^{p - 1}}} {{f_{ - 1}}\left( {{x_{ - 1}}} \right)} \left\{ {\sum\limits_{i = 1}^3 {T_i^{\left(  +  \right)}} } \right\}d{x_{ - 1}} \hfill \\
   = {v^T}\left( { - E\left[ {\delta \left( {1 - {{\tilde Y}^{\left( \ell  \right)}}{{\tilde X}^T}{\theta _{0,\ell }}} \right)\hat X{{\hat X}^T}\left| {{{\tilde Y}^{\left( \ell  \right)}} = 1} \right.} \right]} \right)v \hfill \\
  \begin{array}{*{20}{c}}
  {}&\begin{gathered}
   + O\left( 1 \right)\int_{{R^{p - 1}}} {{f_{ - 1}}\left( {{x_{ - 1}}} \right)\int_{ - 1}^1 {{{\left( {{v_0} + {v_1} + v_{ - 1}^T{x_{ - 1}}} \right)}^2}\left\{ {\left( {1 - {\alpha _{k,0}} - x_{ - 1}^T{\alpha _{k, - 1}} - hy} \right)/{\alpha _{k,1}} - } \right.} }  \hfill \\
  \left. {\left( {1 - {\theta _{0,\ell ,0}} - x_{ - 1}^T{\theta _{0,\ell , - 1}}} \right)/{\theta _{0,\ell ,1}}} \right\}\left| {H'\left( y \right)} \right|dyd{x_{ - 1}} \hfill \\
\end{gathered}
\end{array} \hfill \\
  \begin{array}{*{20}{c}}
  {}&\begin{gathered}
   + O\left( 1 \right)\left\{ {\left( {{\alpha _{k,1}} - {\theta _{0,\ell ,1}}} \right)/{\alpha _{k,1}}{\theta _{0,\ell ,1}}} \right\}\int_{{R^{p - 1}}} {{f_{ - 1}}\left( {{x_{ - 1}}} \right)\int_{ - 1}^1 {\left\{ {{v_0} + } \right.{v_1}\left( {1 - {\theta _{0,\ell ,0}} - x_{ - 1}^T{\theta _{0,\ell , - 1}}} \right)/{\theta _{0,\ell ,1}} + } }  \hfill \\
  {\left. {v_{ - 1}^T{x_{ - 1}}} \right\}^2}\left| {H'\left( y \right)} \right|dyd{x_{ - 1}}。 \hfill \\
\end{gathered}
\end{array} \hfill \\
\end{gathered}
\end{equation}
Notice that
\begin{equation}
\begin{gathered}
  \left| {\left( {1 - {\alpha _{k,0}} - x_{ - 1}^T{\alpha _{k, - 1}} - hy} \right)/{\alpha _{k,1}} - \left( {1 - {\theta _{0,\ell ,0}} - x_{ - 1}^T{\theta _{0,\ell , - 1}}} \right)/{\theta _{0,\ell ,1}}} \right| \hfill \\
   \leqslant C\left( {h + \left| {x_{ - 1}^T\left( {{\alpha _{k, - 1}} - {\theta _{0,\ell , - 1}}} \right)} \right| + \left| {1 - {\alpha _{k,0}} - x_{ - 1}^T{\theta _{0,\ell , - 1}}} \right|\left| {{\alpha _{k,1}} - {\theta _{0,\ell ,1}}} \right| + \left| {{\alpha _{k,0}} - {\theta _{0,\ell ,0}}} \right|} \right). \hfill \\
\end{gathered}
\end{equation}
Hence, we could obtain
\begin{equation}
{T^{\left(  +  \right)}} = {v^T}\left( { - E\left[ {\delta \left( {1 - {{\tilde Y}^{\left( \ell  \right)}}{{\hat X}^T}{\theta _{0,\ell }}} \right)\hat X{{\hat X}^T}\left| {{{\tilde Y}^{\left( \ell  \right)}} = 1} \right.} \right]} \right)v + O\left( {h + {{\left\| {{\alpha _k} - {\theta _{0,\ell }}} \right\|}_2}} \right).
\end{equation}
Conducting the similar procedure on the ${T^{\left(  -  \right)}}$ and according to the \eqref{eh-vev} and the constraint that ${\left\| {{\alpha _k} - {\theta _{0,\ell }}} \right\|_2} = O\left( {{m^{ - 1/2}}} \right)$, we can directly obtain the conclusion of lemma\ref{lemma: eh-vev}.
\end{proof}

\begin{proof}[Proof of Lemma\ref{lemma: tet-discrete}]
Similar with the proof of lemma\ref{lemma: heh-discrete}, we construct a set of vectors $\left\{ {{\alpha _k},1 \leqslant k \leqslant {n^{M\left( {p + 1} \right)}}} \right\}$ in ${R^{p + 1}}$ by dividing each $\left[ {{\theta _{0,{\ell _i}}} - {m^{ - 1/2}},{\theta _{0,{\ell _i}}} + {m^{ - 1/2}}} \right]$ into ${n^M}$ small equal pieces. Hence, for any possible vectors $\alpha $ in the ball ${\left\| {\alpha  - {\theta _{0,\ell }}} \right\|_2} \leqslant O\left( {{m^{ - 1/2}}} \right)$, there exist $\Lambda  \subset \left[ {{n^{M\left( {p + 1} \right)}}} \right]$ such that for any $k \in \Lambda $, we have ${\left\| {\alpha  - {\alpha _k}} \right\|_2} \leqslant 2{\left( {p + 1} \right)^{1/2}}{m^{ - 1/2}}/{n^M}$. Then, with the triangle inequality and Cauchy-Schwarz inequality, it should be easy to verify that
\begin{equation}
\label{tet-discrete}
\begin{gathered}
  \mathop {\sup }\limits_j \mathop {\sup }\limits_{\alpha  - {\theta _{0,\ell }} = {O_P}\left( {{m^{ - 1/2}}} \right)} {T_{3,j}}\left( \alpha  \right) - \mathop {\sup }\limits_j \mathop {\sup }\limits_{k \in \Lambda } {T_{3,j}}\left( {{\alpha _k}} \right) \hfill \\
   \leqslant \mathop {\sup }\limits_j \mathop {\sup }\limits_{\alpha  - {\theta _{0,\ell }} = {O_P}\left( {{m^{ - 1/2}}} \right)} \mathop {\sup }\limits_{k \in \Lambda } {n^{ - 1}}\left\| {v_j^T} \right\|\sum\limits_{i = 1}^n {{{\left\| {\hat X_i^T\tilde Y_i^{\left( \ell  \right)}} \right\|}_2}\left( {T_i^{\left( 1 \right)} + T_i^{\left( 2 \right)}} \right)}  \hfill \\
\end{gathered}
\end{equation}
where
\[\left\{ \begin{gathered}
  T_i^{\left( 1 \right)} = \left| {{\varepsilon _i}H'\left\{ {g\left( {{{\hat X}_i},\tilde Y_i^{\left( \ell  \right)},\alpha } \right)/h} \right\}/h - {\varepsilon _i}H'\left\{ {g\left( {{{\hat X}_i},\tilde Y_i^{\left( \ell  \right)},{\alpha _k}} \right)/h} \right\}/h} \right| \hfill \\
  T_i^{\left( 2 \right)} = \left| {H\left\{ {g\left( {{{\hat X}_i},\tilde Y_i^{\left( \ell  \right)},\alpha } \right)/h} \right\} - H\left\{ {g\left( {{{\hat X}_i},\tilde Y_i^{\left( \ell  \right)},{\alpha _k}} \right)/h} \right\}} \right| \hfill \\
\end{gathered}.  \right.\]
And further, due to the Lipschitzness properties assumed on $H$, we have
\[\begin{gathered}
  T_i^{\left( 1 \right)} \leqslant {C_1}{\left\| {\hat X_i^T\tilde Y_i^{\left( \ell  \right)}} \right\|_2}{\left\| {\alpha  - {\alpha _k}} \right\|_2}/h + {C_2}\left\| {\hat X_i^T\tilde Y_i^{\left( \ell  \right)}} \right\|_2^2{\left\| {\alpha  - {\theta _{0,\ell }}} \right\|_2}{\left\| {\alpha  - {\alpha _k}} \right\|_2}/{h^2} \hfill \\
  T_i^{\left( 2 \right)} \leqslant C\left| {g\left( {{{\hat X}_i},\tilde Y_i^{\left( \ell  \right)},\alpha } \right)/h - g\left( {{{\hat X}_i},\tilde Y_i^{\left( \ell  \right)},{\alpha _k}} \right)/h} \right| \leqslant C{\left\| {\hat X_i^T\tilde Y_i^{\left( \ell  \right)}} \right\|_2}{\left\| {\alpha  - {\alpha _k}} \right\|_2}/h. \hfill \\
\end{gathered} \]
And final result can be obtained by plugging the two inequalities above into \eqref{tet-discrete}.
\end{proof}

\begin{proof}[Proof of Lemma\ref{lemma: t-et}]
With a denotation
\[{\xi _k} = v_k^T\tilde Y_k^{\left( \ell  \right)}{{\hat X}_k}\left[ {H\left\{ {g\left( {{{\hat X}_k},\tilde Y_k^{\left( \ell  \right)},{\alpha _k}} \right)/h} \right\} - I\left( {{\varepsilon _k}} \right) + {\varepsilon _k}H'\left\{ {g\left( {{{\hat X}_k},\tilde Y_k^{\left( \ell  \right)},{\alpha _k}} \right)/h} \right\}/h} \right]\]
and the fact that
\[\begin{gathered}
  \left| {H\left\{ {g\left( {{{\hat X}_i},\tilde Y_i^{\left( \ell  \right)},\alpha } \right)/h} \right\} - I\left( {{\varepsilon _i}} \right) + {\varepsilon _i}H'\left\{ {g\left( {{{\hat X}_k},\tilde Y_k^{\left( \ell  \right)},\alpha } \right)/h} \right\}/h} \right| \hfill \\
   \leqslant C\left( {1 + \left| {\hat X_i^T\left( {\alpha  - {\theta _{0,\ell }}} \right)} \right|/{{\left\| {\alpha  - {{\theta _{0,\ell }}} } \right\|}_2}} \right), \hfill \\
\end{gathered} \]
for a constant $C$, we can assert
\begin{equation}
\sum\limits_{k = 1}^n {E\xi _k^2{e^{t\left| {{\xi _k}} \right|}} \leqslant \sum\limits_{k = 1}^n {E\xi _k^2\exp \left\{ {Ct\left| {v_k^T{{\hat X}_k}} \right|\left( {1 + \left| {\hat X_k^T\left( {\alpha  - {\theta _{0,\ell }}} \right)} \right|/{{\left\| {\alpha  - {{\theta _{0,\ell }}} } \right\|}_2}} \right)} \right\}} } .
\end{equation}
And we can similarly decompose the multivariate integral into a double integral that,
\begin{equation}
\begin{gathered}
  E{\xi ^2}\exp \left\{ {Ct\left| {{v^T}\hat X} \right|\left( {1 + \left| {{{\hat X}^T}\left( {\alpha  - {\theta _{0,\ell }}} \right)} \right|/{{\left\| {\alpha  - {\theta _{0,\ell }}} \right\|}_2}} \right)} \right\} \hfill \\
   \leqslant \int_{{R^p}} {{{\left( {{v^T}\hat x} \right)}^2}\exp \left( {Ct\left| {{v^T}\hat x} \right|\left( {1 + \left| {{{\hat x}^T}\left( {\alpha  - {\theta _{0,\ell }}} \right)} \right|/{{\left\| {\alpha  - {\theta _{0,\ell }}} \right\|}_2}} \right)} \right)} \left[ {H\left\{ {\left( {1 - {{\hat x}^T}\alpha } \right)/h} \right\} - } \right. \hfill \\
  \begin{array}{*{20}{c}}
  {}&{}&{\left. {I\left( \varepsilon  \right) + \varepsilon H'\left\{ {\left( {1 - {{\hat x}^T}\alpha } \right)/h} \right\}/h} \right]f\left( x \right)dx}
\end{array} \hfill \\
  \mathop  = \limits^{\left( a \right)} \left( { - h/{\alpha _1}} \right)\int_{{R^{p - 1}}} {{f_{ - 1}}\left( {{x_{ - 1}}} \right)\int_R {{{\left( {{v^T}\hat x} \right)}^2}\exp \left( {Ct\left| {{v^T}\hat x} \right|\left( {1 + \left| {{{\hat x}^T}\left( {\alpha  - {\theta _{0,\ell }}} \right)} \right|/{{\left\| {\alpha  - {\theta _{0,\ell }}} \right\|}_2}} \right)} \right) \times } }  \hfill \\
  \begin{array}{*{20}{c}}
  {}&{}
\end{array}{\left[ \begin{gathered}
  H\left( z \right) - I\left( {1 - {\theta _{0,\ell ,0}} - \left( {1 - {\alpha _0} - x_{ - 1}^T{\alpha _{ - 1}} - hz} \right)/{\alpha _1} - x_{ - 1}^T{\theta _{0,\ell , - 1}}} \right) +  \hfill \\
  {\theta _{0,\ell ,1}}\left\{ {\left( {1 - {\alpha _0} - x_{ - 1}^T{\alpha _{ - 1}} - hz} \right)/{\alpha _1} - \left( {1 - {\theta _{0,\ell ,0}} - x_{ - 1}^T{\theta _{0,\ell , - 1}}} \right)/{\theta _{0,\ell ,1}}} \right\}H'\left( z \right)/h \hfill \\
\end{gathered}  \right]^2} \times  \hfill \\
  \begin{array}{*{20}{c}}
  {}&{}
\end{array}f\left\{ {\left( {1 - {\alpha _0} - x_{ - 1}^T{\alpha _{ - 1}} - hz} \right)/{\alpha _1}\left| {{x_{ - 1}}} \right.} \right\}dzd{x_1} \hfill \\
   = O\left( {h + {{\left\| {\alpha  - {\theta _{0,\ell }}} \right\|}_2} + \left\| {\alpha  - {\theta _{0,\ell }}} \right\|_2^2/h} \right), \hfill \\
\end{gathered}
\end{equation}
where the variable transformation $z = \left( {1 - {{\hat x}^T}\alpha } \right)/h$ is used in the (a). And finally, adopting the Lemma 1 in \citep{Cai2011Adaptive}, we can obtain
\begin{equation}
 \mathop {\sup }\limits_j \mathop {\sup }\limits_{k \leqslant {n^{C'}}} P\left(\left| {{T_{3,j}}\left( {{\alpha _k}} \right) - E{T_{3,j}}\left( {{\alpha _k}} \right)} \right|\geq C{\left( {\log n/nh} \right)}^{1/2}\right) = {O}\left( n^{-\gamma}\right)
\end{equation}

\end{proof}

\begin{proof}[Proof of Lemma\ref{lemma: et}]
Notice that
\begin{equation}
\label{et-0}
\begin{gathered}
  E{T_{3,j}}\left( {{\alpha _k}} \right) \hfill \\
   = E\left( {{n^{ - 1}}\sum\limits_{i = 1}^n {v_i^T\tilde Y_i^{\left( \ell  \right)}{{\hat X}_i}\left[ {H\left\{ {g\left( {{{\hat X}_i},\tilde Y_i^{\left( \ell  \right)},{\alpha _k}} \right)/h} \right\} - I\left( {{\varepsilon _i}} \right) + {\varepsilon _i}H'\left\{ {g\left( {{{\hat X}_i},\tilde Y_i^{\left( \ell  \right)},{\alpha _k}} \right)/h} \right\}/h} \right]} } \right) \hfill \\
   = E\left( {{v^T}{{\tilde Y}^{\left( \ell  \right)}}\hat X\left[ {H\left\{ {g\left( {\hat X,{{\tilde Y}^{\left( \ell  \right)}},{\alpha _k}} \right)/h} \right\} - I\left( \varepsilon  \right) + \varepsilon H'\left\{ {g\left( {\hat X,{{\tilde Y}^{\left( \ell  \right)}},{\alpha _k}} \right)/h} \right\}/h} \right]} \right) \hfill \\
   = E\left( {{v^T}{{\tilde Y}^{\left( \ell  \right)}}\hat XH\left\{ {g\left( {\hat X,{{\tilde Y}^{\left( \ell  \right)}},{\alpha _k}} \right)/h} \right\}} \right) + E\left( {{v^T}{{\tilde Y}^{\left( \ell  \right)}}\hat XI\left( \varepsilon  \right)} \right) + E\left( {{v^T}{{\tilde Y}^{\left( \ell  \right)}}\hat X\varepsilon H'\left\{ {g\left( {\hat X,{{\tilde Y}^{\left( \ell  \right)}},{\alpha _k}} \right)/h} \right\}/h} \right) \hfill \\
   = {E_1} + E\left( {{v^T}{{\tilde Y}^{\left( \ell  \right)}}\hat XI\left( \varepsilon  \right)} \right) + {E_2}. \hfill \\
\end{gathered}
\end{equation}
We can then decompose the ${E_1}$ and ${E_2}$ into two parts separately according to the value of $Y$, that is to say
\[\begin{gathered}
  {E_1} = E_1^{\left(  +  \right)}P\left[ {{{\tilde Y}^{\left( \ell  \right)}} = 1} \right] + E_1^{\left(  -  \right)}P\left[ {{{\tilde Y}^{\left( \ell  \right)}} =  - 1} \right] \hfill \\
  \left\{ \begin{gathered}
  E_1^{\left(  +  \right)} = \int_{{R^p}} {\left( {{v_0} + v_{ - 0}^Tx} \right)H\left\{ {\left( {1 - {\alpha _{k,1}} - {x^T}{\alpha _{k, - 1}}} \right)/h} \right\}f\left( x \right)dx}  \hfill \\
  E_1^{\left(  -  \right)} = \int_{{R^p}} {\left( {{v_0} + v_{ - 0}^Tx} \right)H\left\{ {\left( {1 + {\alpha _{k,1}} + {x^T}{\alpha _{k, - 1}}} \right)/h} \right\}f\left( x \right)dx}  \hfill \\
\end{gathered}  \right. \hfill \\
\end{gathered} \]
and
\[\begin{gathered}
  {E_2} = E_2^{\left(  +  \right)}P\left[ {{{\tilde Y}^{\left( \ell  \right)}} = 1} \right] + E_2^{\left(  -  \right)}P\left[ {{{\tilde Y}^{\left( \ell  \right)}} =  - 1} \right] \hfill \\
  \left\{ \begin{gathered}
  E_2^{\left(  +  \right)} = \int_{{R^p}} {\left( {{v_0} + v_{ - 0}^Tx} \right)\left\{ {\left( {1 - {\theta _{0,\ell ,1}} - {x^T}{\theta _{0,\ell , - 1}}} \right)/h} \right\}H\left\{ {\left( {1 - {\alpha _{k,1}} - {x^T}{\alpha _{k, - 1}}} \right)/h} \right\}f\left( x \right)dx}  \hfill \\
  E_2^{\left(  -  \right)} = \int_{{R^p}} {\left( {{v_0} + v_{ - 0}^Tx} \right)\left\{ {\left( {1 + {\theta _{0,\ell ,1}} + {x^T}{\theta _{0,\ell , - 1}}} \right)/h} \right\}H\left\{ {\left( {1 + {\alpha _{k,1}} + {x^T}{\alpha _{k, - 1}}} \right)/h} \right\}f\left( x \right)dx}.  \hfill \\
\end{gathered}  \right. \hfill \\
\end{gathered} \]
And finally, by expanding these integrals to the joint integrals of $\left( {{x_1},{x_{ - 1}}} \right)$ as what we have conducted in the proof of lemma\ref{lemma: eh-vev}, we can assert
\begin{equation}
\label{e1e2}
\begin{gathered}
  E_1^{\left(  +  \right)} + E_2^{\left(  +  \right)} =  - E\left( {{v^T}{{\tilde Y}^{\left( \ell  \right)}}\hat XI\left( \varepsilon  \right)\left| {{{\tilde Y}^{\left( \ell  \right)}} = 1} \right.} \right) + O\left( {h + \left\| {{\alpha _k} - {\theta _{0,\ell }}} \right\|_2^2} \right) \hfill \\
  E_1^{\left(  -  \right)} + E_2^{\left(  -  \right)} =  - E\left( {{v^T}{{\tilde Y}^{\left( \ell  \right)}}\hat XI\left( \varepsilon  \right)\left| {{{\tilde Y}^{\left( \ell  \right)}} =  - 1} \right.} \right) + O\left( {h + \left\| {{\alpha _k} - {\theta _{0,\ell }}} \right\|_2^2} \right). \hfill \\
\end{gathered}
\end{equation}
Finally, plugging \eqref{e1e2} into \eqref{et-0}, we have
\begin{equation}
\begin{gathered}
  \mathop {\sup }\limits_j \mathop {\sup }\limits_{1\le k \le n^{C'}}E{T_{3,j}}\left( {{\alpha _k}} \right) = E\left( {{v^T}{{\tilde Y}^{\left( \ell  \right)}}\hat XI\left( \varepsilon  \right)} \right) - E\left( {{v^T}{{\tilde Y}^{\left( \ell  \right)}}\hat XI\left( \varepsilon  \right)} \right) + O\left( {h + \left\| {{\alpha _k} - {\theta _{0,\ell }}} \right\|_2^2} \right) \hfill \\
  \begin{array}{*{20}{c}}
  {}&{}
\end{array} = O\left( {h + \left\| {{\alpha _k} - {\theta _{0,\ell }}} \right\|_2^2} \right) = O\left( {h + {m^{ - 1}}} \right) \hfill \\
\end{gathered}
\end{equation}
\end{proof}{}

\bibliographystyle{asa}

\end{document}